\documentclass[runningheads,a4paper]{article}
\usepackage{amssymb,amsmath,amsthm}
\usepackage{color}
\usepackage[mathcal]{euscript}
\usepackage{microtype}
\usepackage{array}
\usepackage{graphicx}
\usepackage[margin=1in]{geometry}
\usepackage{hyperref}
\usepackage[hang,flushmargin,bottom]{footmisc}
\usepackage[arxiv]{optional}

\graphicspath{{./fig/}}

\definecolor{gray}{rgb}{0.35,0.35,0.35}
\definecolor{blue}{rgb}{0,0,1}
\definecolor{red}{rgb}{1,0,0}
\definecolor{orange}{rgb}{0.75, 0.4, 0}
\definecolor{green}{rgb}{0.0, 0.5, 0.0}

\newcommand{\tzvika}[1]{{\color{blue}\textbf{Tzvika: }\sf#1}}

\newtheorem*{proposition*}{proposition}

\newcommand{\ignore}[1]{}

\def\F{\mathcal{F}}  
\def\O{\mathcal{O}} \def\I{\mathcal{I}} \def\S{\mathcal{S}}
  \def\I{\mathcal{I}}

\def\W{\mathcal{W}} \def\R{\mathcal{R}} 
 \def\A{\mathcal{A}}

\def\eps{\varepsilon}

\renewcommand{\leq}{\leqslant}
\renewcommand{\geq}{\geqslant}

\newcommand{\Cpp}{C\raise.08ex\hbox{\tt ++}\xspace}

\newcommand{\argmin}{\operatornamewithlimits{argmin}}

\def\0{\bm{0}}

\makeatletter
\def\thmhead@plain#1#2#3{%
  \thmname{#1}\thmnumber{\@ifnotempty{#1}{ }\@upn{#2}}%
  \thmnote{ {\the\thm@notefont#3}}}
\let\thmhead\thmhead@plain
\makeatother

\def\eps{\varepsilon}

\def\reals{\mathbb{R}}

\def\seq#1{\langle #1 \rangle}

\def\abs#1{\mathopen| #1 \mathclose|}		
\def\norm#1{\mathopen\| #1 \mathclose\|}	

\def\Paren#1{\left( #1 \right)}		

\usepackage{textcomp}
\DeclareRobustCommand{\cost}{\textit{\raisebox{0.27pt}{\scalebox{0.97}{\textcent}}}}

\usepackage{accents}

\def\interior{{\text{int}}}

\newcommand{\poslit}{\ensuremath{\R^+}}
\newcommand{\neglit}{\ensuremath{\R^-}}

\def\ass{\Lambda}
\def\pth{\pi_0}

\newcommand{\starts}{\mathbf{s}}
\newcommand{\finals}{\mathbf{f}}
\newcommand{\pathens}{\Pi}
\newcommand{\shrtpath}{\Gamma}
\newcommand{\optc}{\cost^*(\I)}

\title{Multi-Robot Motion Planning for Unit Discs with Revolving Areas\thanks{Work by Pankaj K. Agarwal and Erin Taylor is supported by IIS-1814493, CCF-2007556, and CCF-2223870. 
Work by Dan Halperin and Tzvika Geft has been supported in part by the Israel Science Foundation (grant no.~1736/19), by NSF/US-Israel-BSF (grant no.~2019754), by the Israel Ministry of Science and Technology (grant no.~103129), by the Blavatnik Computer Science Research Fund, and by the Yandex Machine Learning Initiative for Machine Learning
at Tel Aviv University. Tzvika Geft has also been supported by a scholarship from the Shlomo Shmeltzer Institute for Smart Transportation at Tel Aviv University.}}
\author {
    Pankaj K.~Agarwal \thanks{Department of Computer Science, Duke University, USA.} \\ \small{pankaj@cs.duke.edu}  \and 
    Tzvika Geft\thanks{School of Computer Science, Tel Aviv University, Israel.} \\ \small{zvigreg@mail.tau.ac.il} \and 
    Dan Halperin\footnotemark[3] \\ \small{danha@tauex.tau.ac.il} \and 
    Erin Taylor\footnotemark[2] \\ \small{ect15@cs.duke.edu}
}

\makeatletter
\renewcommand\subparagraph{%
 \@startsection {subparagraph}{5}{\z@ }{3.25ex \@plus 1ex
 \@minus .2ex}{-1em}{\normalfont \normalsize \bfseries }}%
\makeatother

\newtheorem{theorem}{Theorem}
\newtheorem{corollary}[theorem]{Corollary}
\newtheorem{lemma}[theorem]{Lemma}
\theoremstyle{definition}

\usepackage{amsthm}
\usepackage{mathtools}
\usepackage[dvipsnames,usenames]{xcolor}
\usepackage{hyperref}
\hypersetup{colorlinks=true, urlcolor=Blue, citecolor=Green, linkcolor=Black, breaklinks, unicode}

\usepackage{float}
\usepackage{graphicx}
\graphicspath{{figures/}} 
\DeclareGraphicsExtensions{.pdf,.png,.jpg}

\usepackage{mathdesign}

\usepackage[nocompress]{cite}
\usepackage{enumerate}
\usepackage{paralist}
\usepackage{bm}

\usepackage{caption}
\usepackage{subcaption}
\usepackage{thm-restate}

\def\eat#1{#1}
\def\short#1{}

\begin{document}
\maketitle

\begin{abstract}
We study the problem of motion planning for a collection of $n$ labeled unit disc robots in a polygonal environment. 
We assume that the robots have \emph{revolving areas} around their start and final positions: that each start and each final is contained in a radius $2$ disc lying in the free space, not necessarily concentric with the start or final position, which is free from other start or final positions.  
This assumption allows a \emph{weakly-monotone} motion plan, in which robots move according to an ordering as follows: during the turn of a robot $R$ in the ordering, it moves fully from its start to final position, while other robots do not leave their revolving areas.
As $R$ passes through a revolving area, a robot $R'$ that is inside this area may move within the revolving area to avoid a collision.
Notwithstanding the existence of a motion plan, we show that minimizing the total traveled distance in this setting, specifically even when the motion plan is restricted to be weakly-monotone, is APX-hard, ruling out any polynomial-time $(1+\eps)$-approximation algorithm. 

On the positive side, we present the first constant-factor approximation algorithm for computing a feasible weakly-monotone motion plan. The total distance traveled by the robots is within an $O(1)$ factor of that of the optimal motion plan, which need not be weakly monotone. 
Our algorithm extends to an online setting in which the polygonal environment is fixed but the initial and final positions of robots are specified in an online manner. 
Finally, we observe that the overhead in the overall cost that we add while editing the paths to avoid robot-robot collision can vary significantly depending on the ordering we chose. Finding the best ordering in this respect is known to be NP-hard, and we provide a polynomial time $O(\log n \log \log n)$-approximation algorithm for this problem.
\end{abstract}

\section{Introduction} 

Multi-robot systems are already in use in logistics, in a variety of civil engineering and nature preserving tasks, and in agriculture, to name a few areas. They are anticipated to proliferate in the coming years, and accordingly they attract intensive research efforts in diverse communities.

A basic motion-planning problem for a team of robots is to plan such collision-free paths for the robots between given start and final positions. Among the many dimensions along which the multi-robot motion planning (MRMP) problem has been studied, we focus on three: (1) we distinguish between distributed and centralized control. In the former each robot has limited knowledge of the entire environment where the robots move, and each robot may communicate with few neighboring robots. In the latter, which is typical in factory automation and other well-structured environments, a central authority has control over all the robots and the planning for each robot takes into consideration knowledge about the state of all the other robots in the system. (2) In the \emph{labeled} version the robots are distinguishable from one another and each robot has its own assigned target, whereas in the \emph{unlabeled} version the robots are indistinguishable, i.e., each target can be occupied by any robot in the team and the motion-planning problem is considered solved if at the end of the motion all the target positions are occupied. (3) We further distinguish between \emph{continuous} or \emph{discrete} domains. Much of the study of motion planning in computational geometry and robotics assumes that the workspace is \emph{continuous}. In AI research, where the problem is typically called multi-agent path finding (MAPF)~\cite{stern2019multiagent}, the domain is modeled as a graph. Nowadays the MAPF problem is studied in diverse research communities, often as an approximation of the continuous domain.

In our study here we consider a \emph{centralized, labeled}, and \emph{continuous} version of MRMP.  Furthermore, we are not only interested in finding a solution to the given  motion-planning problem, but rather in finding a high-quality solution. Specifically, we aim to find a solution that minimizes the total path length traveled by the robots. 

\subparagraph{Related Work.}
Computing a feasible motion plan (not necessarily a good one) itself is in general computationally hard for MRMP (see, e.g., \cite{hopcroft1984complexity, DBLP:conf/fun/BrunnerCDHHSZ21, DBLP:journals/ijrr/SoloveyH16,geft2021complexity}).  In the results that we cite next, some additional mitigating conditions are assumed on the system to obtain efficient motion-planning algorithms. 

There are  few results that guarantee bounds on the quality of the motion plans for multi-robot systems. For complete algorithms\footnote{A motion planning algorithm is called \emph{complete} if, in finite time, it is guaranteed to find a solution or determine that no solution exists.} in the unlabeled case, there are bounds on the length of the longest path taken by a robot in the system \cite{DBLP:journals/arobots/TurpinMMK14}, or on the sum of distance traveled by all the robots~\cite{DBLP:conf/rss/SoloveyYZH15}. 
For the labeled case, Demaine et al.~\cite{demaine2019coordinated} provide constant-factor approximation algorithms for minimizing the execution time of a coordinated parallel motion if there are no obstacles. 
Still for the labeled case, Solomon and Halperin obtained a very crude bound on the sum of distances~\cite{SolomonHalperin2018} (the approximation factor can be linear in the complexity of the environment in the worst case) in a setting identical to the setting of the current paper, namely assuming the existence of \emph{revolving areas}---see below for a formal definition. 
No sublinear approximation algorithm is known for MRMP even if we assume the existence of revolving areas and the cost of a motion plan is the sum of the lengths of individual paths. In the current paper we significantly improve over and expand the results in~\cite{SolomonHalperin2018} in several ways, as we discuss below.  

An alternative approach to cope with the hardness of motion planning is to use \emph{sampling-based} methods~\cite{DBLP:journals/cacm/Salzman19}. In their seminal paper, Karaman and Frazzoli~\cite{karaman2011sampling} (see also~\cite{DBLP:conf/icra/SoloveyJSFP20}) introduced an algorithm, called RRT*, which guarantees near optimality if the number of samples tends to infinity. A related algorithm dRRT* handles the multi-robot case with the same type of guarantee~\cite{DBLP:journals/arobots/ShomeSDHB20}. Recently Dayan et~al.~\cite{DBLP:conf/icra/DayanSPH21} have obtained near-optimality with finite sample size for the multi-robot case.

\subparagraph{Problem Statement.} Let $\W$ be a polygonal environment, that is, a polygon with holes in $\reals^2$ and a total of $m$ vertices. Let $R_1, \dots, R_n$ be $n$ robots, each modeled as a unit disc, that move around in $\W$.
Let $\mathcal{O} = \reals^2 \setminus \W$ be the obstacle space.
For a point $p \in \reals^2$, let $D_p$ denote the unit disc centered at point $p$. 
Let $\F = \{x \in \W : D_x \cap \O = \emptyset \}$ represent the free space of $\W$ (with respect to one $R_i$). 
A \emph{path} is a continuous function $\mathbf{\pi}: I \rightarrow \reals^2$ from an interval $I$ to $\reals^2$, and is \emph{collision-free} if it is contained in $\F$. 
Let $\ell(\pi)$ denote the arc length of $\pi$, i.e., $\ell(\pi) = \int_I \abs{\pi'(t)} dt$.
The position of each $R_i$ is specified by the $x$- and $y$-coordinates of its center $c_i$ and we use $R_i(c)$ to denote $R_i$ being at $c$ (note that $R_i(c)$ is the same as $D_c$), and a motion of $R_i$ is specified by the path followed by its center. 
Let $\interior\  D$ denote the interior of disc $D$.  
A \emph{path ensemble} $\pathens = \{ \pi_1, \dots,
\pi_n\}$ 
is a set of $n$ paths defined over a common interval $I$, i.e. $\pi_i: I \rightarrow \reals^2$, for $1 \leq i \leq n$; $\pathens$ is called \emph{feasible} if (i) $\pi_i \subset \F$ for every $i \leq n$, and (ii) for any $t \in I$ and for any pair $i \neq j$, $\interior\  R_i(\pi_i(t)) \cap \interior\  R_j(\pi_j(t)) = \emptyset$, i.e., the $R_i$'s remain in $\W$ and they do not \emph{collide} with each other (but may touch each other) during the entire motion. 
We also refer to $\pathens$ as a \emph{motion plan} of $R_1, \dots, R_n$. 
The cost of $\pathens$, denoted by $\cost(\pathens)$, is defined as $\cost(\pathens) = \sum_{i =1}^n \ell(\pi_i)$.

We are given a set of \emph{start} positions $\starts = \{s_1, \dots, s_n \}$ where the $n$ robots initially lie  and a set of \emph{final} (also called \emph{target}) positions $\finals = \{ f_1, \dots, f_n \}$. 
Our goal is to find a path ensemble $\pathens^* = \{ \pi^*_1, \dots, \pi^*_n\}$ over an interval $[0, T]$ where $T$ denotes the ending time of the last robot movement, 
\begin{description} 
\item[(i)] $\pi^*_i(0) = s_i$ and $\pi^*_i(T) = f_i$ for all $i$, and 
\item[(ii)] $\cost(\pathens^*) = \min_{\pathens} \cost(\pathens)$ where the minimum is taken over all feasible path ensembles.
\end{description}


We refer to the problem as \emph{optimal multi-robot motion planning (MRMP)}.
In this paper, we investigate optimal MRMP under the assumption that there is some free space around the starting and final positions of $R_1, \dots, R_1$, a formulation introduced in~\cite{SolomonHalperin2018}.
A \emph{revolving area} of a start or final position $z \in \starts \cup \finals$, is a disc $A_z$ of radius $2$ such that:
\begin{inparaenum}[(i)]
\item $D_z \subseteq A_z$, 
\item  $A_z \cap \O = \emptyset$, and
\item for any other start or final position $y \in (\starts \cup \finals) \setminus \{z\}$, $A_z \cap D_y = \emptyset$. 
\end{inparaenum}
That is, each $R_i$ lies in a revolving area at its start and final position (note that $z$ need not be the center of the revolving area $A_z)$ and does not intersect any other revolving areas, and the revolving areas do not intersect any obstacles.
We remark that the revolving areas may intersect one another; this makes the separation assumptions in the current paper lighter than in related results (e.g.,\cite{DBLP:journals/tase/AdlerBHS15}), which in turn makes the analysis more involved.
See Figure~\ref{fig:simple-revolving-area} for an example.
Set $\A = \{A_z : z \in \starts \cup \finals \}$.
We refer to this problem as \emph{optimal multi-robot motion planning with Revolving Areas (MRMP-RA)}.

We define the \textit{active interval} $\tau_i \subseteq [0,1]$ as the open interval from the first time $R_i$ leaves the revolving area $A_{s_i}$ of $s_i$ to the last time $R_i$ is not in the revolving area $A_{f_i}$ of $f_i$. If the active intervals $\{\tau_1,\ldots,\tau_n\}$ are pairwise disjoint then we call $\pathens$ a \emph{weakly-monotone motion plan} (with respect to revolving areas).\footnote{We use the term ``weakly-monotone" because a plan is called monotone if the active interval of $R_i$ is defined from when $R_i$ leaves $s_i$ for the first time until $R_i$ reaches $f_i$ for the last time, (rather than the leaving/reaching the revolving area $A_{s_i}$/$A_{f_i}$).} Finally, an instance of optimal MRMP is specified as $\I = (\W, \starts, \finals)$ where 
$\W, \starts, \finals$ are as defined above. 
Let $\pathens^*(\I)$ denote an optimal solution of $\I$ and let $\cost^*(\I) = \cost(\pathens^*(\I))$.

\subparagraph{Our Results.}

\begin{figure}
    \captionsetup[subfigure]{justification=centering}
     \centering
     \begin{subfigure}[b]{0.3\textwidth}
         \centering
         \includegraphics[width=\textwidth]{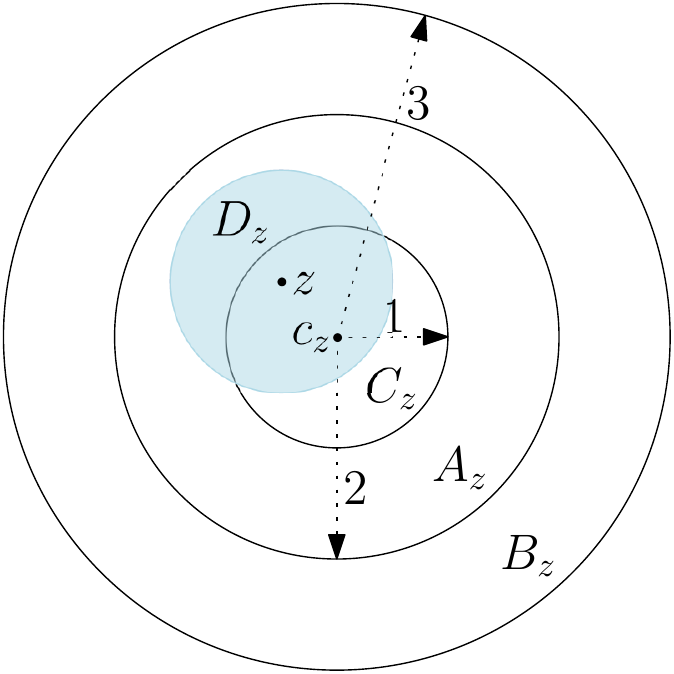}
         \caption{}
     \end{subfigure}
     \hfill
     \begin{subfigure}[b]{0.6\textwidth}
         \centering
         \includegraphics[width=.9\textwidth]{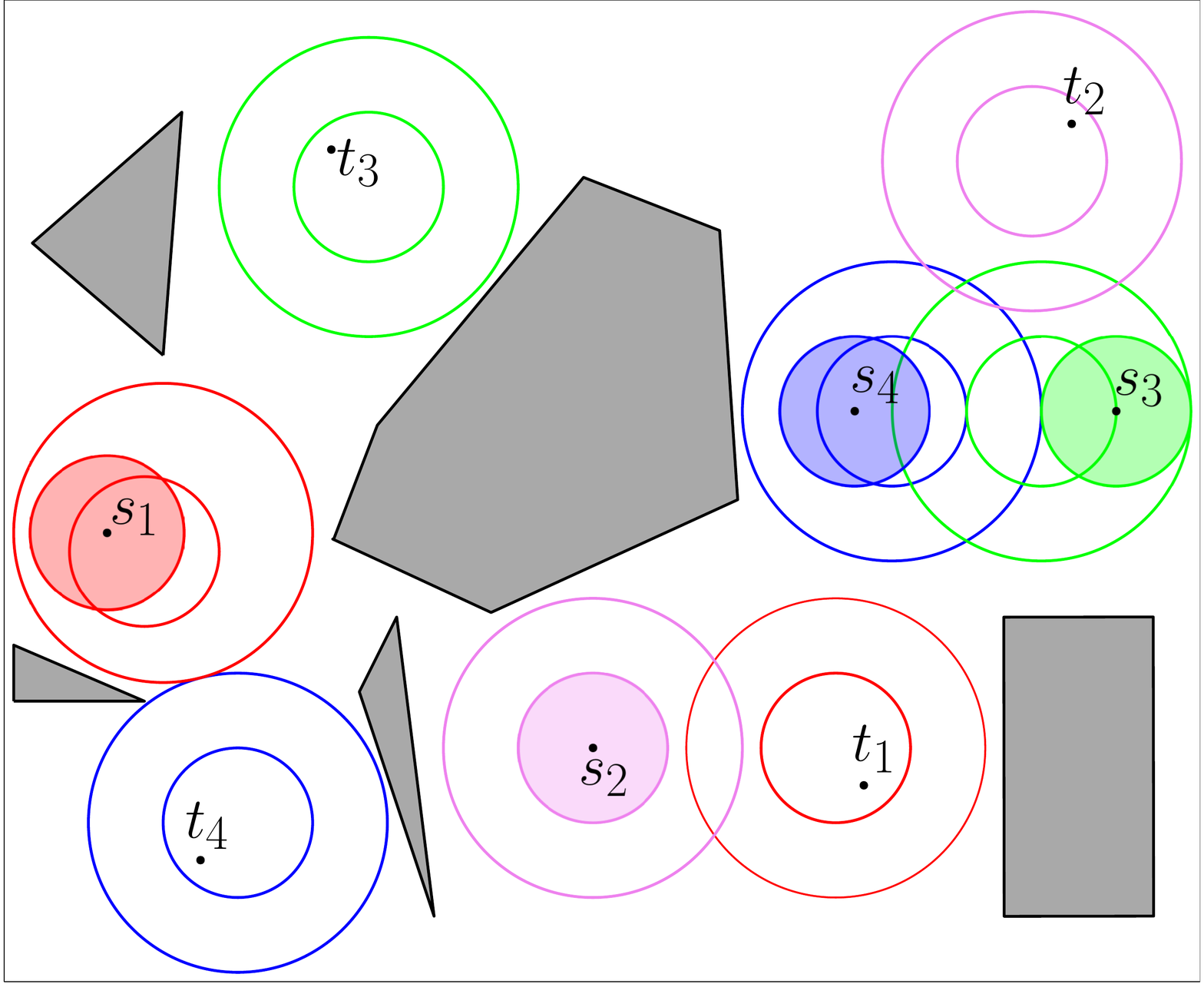}
         \caption{}
     \end{subfigure}
        \caption{(a) On the left, revolving area $A_z$ for some $z \in \starts \cup \finals$ with $D_z \subseteq A_z$, core $C_z$, and buffer $B_z$. (b) On the right, we show an instance of MRMP-RA. Each robot is shown as a filled disc in its starting revolving area, and its target revolving area is shown in the same color. Obstacles are dark gray.}
\label{fig:simple-revolving-area}
\end{figure}

The paper contains the following three main results: 
\begin{description}
\item[(A) Hardness results.] In Section~\ref{sec:np-hardness}, we show that MRMP-RA is NP-hard under the weakly-monotone assumption.
The NP-hardness of optimizing sum of distances (i.e., optimal MRMP) for the monotone and the general (non-monotone) case was shown in \cite{GH2021Refined}, but without revolving areas.
Our main result here is the extension of the NP-hardness construction to prove that MRMP-RA, under the weakly-monotone assumption, is in fact APX-hard, which rules out a polynomial-time $(1+\eps)$-approximation algorithm for it.
To the best of our knowledge, this is the first APX-hardness result for any MRMP variant.
\item[(B) Approximation algorithm.] In Section~\ref{sec:alg}, we present the first $O(1)$-approximation algorithm that given an instance $\I = (\W, \starts, \finals, \A)$ of MRMP-RA computes a feasible path ensemble $\pathens$ from $\starts$ to $\finals$ such that $\pathens$ is weakly-monotone and $\cost(\pathens) = O(1) \cdot \cost^*(\I)$; note that $\Pi^*(\I)$ need not be weakly-monotone, i.e., we approximate the general optimal path ensemble. 
In fact, we show that the robots can be moved in any order, so our algorithm can be extended to an online setting where the robots $R_i$, and their start/final positions, $(s_i, f_i)$ are given in an online manner, or $R_i$'s may have to execute multiple tasks which are given in an online manner-- the so-called life-long planning problem. Our algorithm ensures an $O(1)$ competitive ratio, i.e., the cost is $O(1)$ times the optimal cost of the offline problem.

The algorithm begins by computing a set of shortest paths $\shrtpath$ that avoid obstacles but ignore robot-robot collisions. Then, $\shrtpath$ is edited to avoid robot-robot collisions by moving non-active robots within their revolving areas.
Our overall approach is the same as by Solomon and Halperin~\cite{SolomonHalperin2018}, but the editing of $\shrtpath$ differs significantly from~\cite{SolomonHalperin2018}, so that the cost of the paths does not increase by too much. 
We use a more conservative editing of $\Gamma$, which  enables us to prove that
the cost of the edited path ensemble is $O(1) \cdot \cost(\Gamma)$ (see Section~\ref{sec:analysis}), while the cost of the edited path in~\cite{SolomonHalperin2018} is\footnote{Notice that the roles of $m$ and $n$ here are reversed with respect to~\cite{SolomonHalperin2018}.}  $O(\cost(\Gamma) +mn + m^2)$. 
Our main technical contributions are defining a more conservative retraction, proving that the motion plan remains feasible even under this conservative retraction, and bounding the total cost of the motion plan by using a combination of local and global arguments.  Analyzing both the feasibility and the cost of the motion plan are nontrivial and require new ideas. 
\item[(C) Computing a good ordering.] 
The result above shows that editing the paths increases the total cost of the motion plan only by a constant factor irrespective of the order in which we move the robots. 
However, the overhead in the overall cost due to editing (to avoid robot-robot collisions) can vary significantly depending on the ordering we chose. 
This raises the question whether we can find a ``good" ordering that minimizes the overhead. 
The result in~\cite{SolomonHalperin2018} implies that the problem of finding a good ordering that minimizes the amount of overhead is NP-hard.\footnote{The model in~\cite{SolomonHalperin2018} for defining the overhead is different from ours, their construction can nevertheless be adapted to our setting.} 
We present a polynomial time $O(\log n \log \log n)$-approximation algorithm for finding a good ordering. This is achieved by reducing the problem to an instance of weighted feedback arc set in a directed graph, and applying an approximation algorithm for the latter problem~\cite{even1998approximating}. 
Due to lack of space, this result is described in Section~\ref{sec:ordering}.
\end{description}

We emphasize that without additional, mitigating, assumptions, MRMP is intractable. Sampling-based planners assume that the full solution paths have some clearance around them---namely, each robot has some distance from the obstacles along its entire path, as well as from the other robots. Here, we  assume certain clearance only at the start and goal positions; we do not make any assumption about the clearance along the paths.
Indeed, we assume non-negligible clearance, as we require that each robot at a start or goal position is encapsulated inside a disc of radius~$\boldmath{2}$, which does not contain any other robot at its start or goal position. The choice of the number~$2$ here is not arbitrary. In a couple of related results for MRMP of unit discs~\cite{DBLP:journals/tase/AdlerBHS15, BanyassadyEtAl.SoCG.2022} this is the critical value of clearance below which there does not always exist a solution to the problem. 


\section{Hardness of Distance Optimal MRMP-RA} \label{sec:np-hardness}

In this section we present our hardness results.
Throughout this section all path ensembles are weakly-monotone, unless otherwise stated.
With a slight abuse of notation we use $\cost^*$ to denote the cost of the optimal weakly-monotone path ensemble.
Finding monotone path ensembles has been shown to be NP-hard in~\cite{geft2021complexity} using a similar grid-based construction without revolving areas.


\noindent \subparagraph*{NP-Hardness of weakly-monotone MRMP-RA} 
Let $Q(x_1,\ldots, x_n) = \bigwedge_{i=1}^m C_i$ be an instance of 3SAT with $n$ variables and $m$ clauses. Each clause $C_i$ is a disjunction of three \textit{literals}, which are variables or their negations.
We construct a corresponding MRMP-RA instance $\I \coloneqq \I(Q) = (\W, \starts, \finals, \A)$ with $N = 3m+1$ robots and choose a real value $d \ge 0$ such that $\cost^*(\I) \le d$ if and only if $Q$ is satisfiable.
Let $d(\I) \coloneqq \sum_{i=1}^N d_i$, where $d_i$ is the length of the optimal path of $R_i$ from $s_i$ to $f_i$ in $\W$, ignoring other robots. 
In fact, our construction will choose $d$ to be $d(\I)$, that is, $d$ is the lowest possible cost of a feasible path ensemble from $\starts$ to $\finals$ in $\W$.
Our construction will ensure that the lowest cost is attained if and only if $Q$ is satisfiable.
$\I$ is constructed so that a path ensemble with such a cost is possible if and only if (a feasible) monotone motion plan exists.
An example of the construction is shown in Figure~\ref{fig:hardness}.

\begin{figure*}[ht]
\centering
\includegraphics[width=0.88\textwidth]{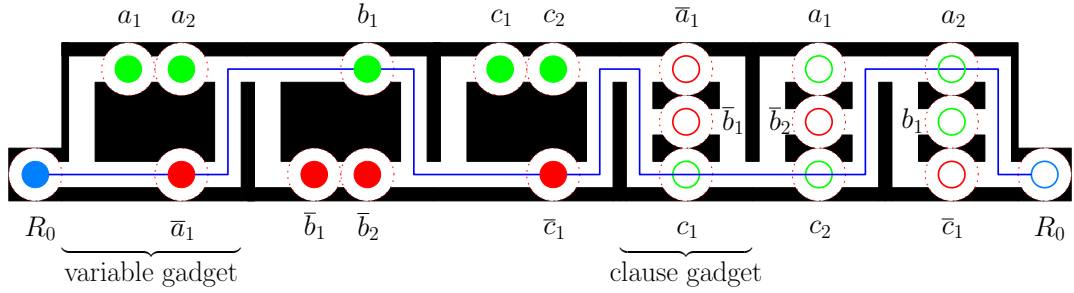}
\caption{The MRMP-RA instance $\I$ that corresponds to the formula $Q = (\overline{a} \lor \overline{b} \lor c) \land (a \lor \overline{b} \lor  c) \land (a \lor b \lor \overline{c})$.
The start and target positions are the filled and unfilled discs, respectively. Positive literal robots are green, negative literal robots are red.
Obstacles appear in black.
Start and target positions of literal robots are labeled with unique indices in order to distinguish between appearances of the same literal. The path $\pth$ is shown in blue for the assignment $a=T, b=F, c=T$ for which the corresponding path ensemble has robots moving in the following order: $\overline{c}_1, b_1, \overline{a}_1, r_0, a_2, a_1, \overline{b}_2, c_2, \overline{b}_1, c_1$.
}
\label{fig:hardness}
\end{figure*} 

\smallskip
\noindent\textbf{\textit{Overall description.}}
%
The workspace $\W$ consists of $m+n$ rectangular gadgets, one for each variable and each clause, referred to as \emph{variable} and \emph{clause} gadgets, respectively.
All the gadgets have unit-width \emph{passages} that are wider around revolving areas.
For simplicity, the widened areas are shown as circular arcs, but they can easily be made polygonal.
Each gadget has an entrance on the left and an exit on the right.
The vertical positions of entrances and exits alternate so that a gadget's entrance is connected to the exit of the gadget on its left.

There are $N = 3m+1$ robots, each being a unit disc: one robot for each appearance of a literal in $Q$, which are collectively called \emph{literal robots}, and one special \emph{pivot} robot $R_{0}$ (shown in blue in Figure~\ref{fig:hardness}).
The robot $R_{0}$ has to pass through all the gadgets from left to right, by which it is able to verify the satisfiability of $Q$, and the literal robots will constrain its motion in order to ensure that $\cost^*(\I) \le d$.

Each variable (resp.clause) gadget contains two (resp. three) horizontal passages, which offer two (resp. three) shortest paths from its entrance to its exit.
Each such path consists of vertical and horizontal line segments.
The horizontal passages of the gadgets contain all the start and target positions of literal robots.
All the revolving areas are centered at their respective start or target positions, and they do not overlap.

\smallskip
\noindent\textbf{\textit{Gadgets.}}
Each variable gadget initially contains robots representing literals of a single variable of $Q$.
The top and bottom horizontal passages of the gadget contain robots representing only positive and negative literals, respectively.
Each clause gadget has three horizontal passages, each containing a target position of one the literals in the corresponding clause.
The gadgets are placed within a horizontal strip from left to right such that variable gadgets are located to the left of clause gadgets.
The order of gadgets of the same type is arbitrary, however it determines the order of the start positions, which is critical: the left to right order of start positions within each variable gadget is set to match the left to right order of the corresponding target positions.
We refer to this order as the \emph{intra-literal order property}.
We say that a revolving area $A$ is \emph{congested} if it contains two robots at the same time. Intuitively, both optimal path ensembles and monotone path ensembles need to prevent revolving areas from becoming congested.
\short{The following lemma is proved in Appendix~\ref{subsec:missing-proofs}.} We first establish that finding an optimal weakly-monotone path ensemble is equivalent to finding a monotone one, then show the equivalence between a satisfying assignment and a monotone path ensemble.

\begin{restatable}{lemma}{monotone} 
\label{lem:monotone-connection}
$\I$ has a weakly monotone path ensemble with a cost of $d$ if and only if $\I$ has a monotone path ensemble.
\end{restatable}

\eat{\begin{proof}
Let $A$ be a revolving area in $\W$.
We first note that without any loss of generality, in any path ensemble of $\I$ a robot may either be contained in $A$ at some point or never intersect $A$ at all. That is, a robot will not partially penetrate a revolving area without ever fully entering it.
\ 
Let $\pathens$ be a feasible path ensemble with $\cost(\pathens) = d$.
We fix a robot $R_i$ and examine the motion that occurs during its active interval $\tau_i$.
We claim that any motion of a robot $R_j, j \neq i$ during $\tau_i$ is \emph{redundant}, i.e., if $R_j$ does not move during $\tau_i$ then $R_i$ can still perform the same motion.
This suffices in order to conclude that $\pathens$ can be made monotone.
Observe that during the execution of $\pathens$ no revolving area $A$ can become congested, as otherwise the two robots that are simultaneously in $A$ will have to take a path that is longer than the shortest path that ignores other robots.
Therefore, whenever $R_i$ is inside a revolving area $A$, it is the only robot in $A$, and any motion by other robots is redundant.
Whenever $R_i$ is not contained in any revolving area, all other robots must be contained in revolving areas, by definition.
Hence, any motion by other robots at such point in time is also redundant.
So overall, $R_i$ may travel along its whole path without other robots moving.
\ 
For the other direction, in a monotone path ensemble it also holds that no revolving area may become congested (as otherwise robots move simultaneously).
Therefore, any revolving area that some robot $R_i$ intersects during its motion must not contain other robots.
For any gadget $g$ that $R_i$ needs to traverse, this allows $R_i$ to take some shortest path through $g$. Therefore, $R_i$ is able to take the shortest path that ignores other robots overall.
Hence, a path ensemble with a cost of $d$ exists.
\end{proof}}
\begin{restatable}{lemma}{assignment} \label{thm:monotone-hardness}
$Q$ has a satisfying assignment if and only if $\I$ has a monotone path ensemble.
\end{restatable}
\begin{proof}
Assume that $Q$ has a satisfying assignment $\ass$.
Let \poslit{} (resp. \neglit{}) denote the set of robots corresponding to literals that evaluate to true (resp. false) according to $\ass$.
That is, for each variable gadget, \poslit{} contains robots that are all initially either in the top or the bottom passage, according to $\ass$.
We show that the robots can move along optimal paths in the order \neglit{}, $R_{0}$, \poslit{}, which is made precise below.

Let $\pi_0$ be a shortest collision-free path from $s_{0}$ to $f_{0}$ that passes only through the start positions of \neglit{} and targets of \poslit{}; see Figure~\ref{fig:hardness}.
The path $\pi_0$ exists because each clause gadget must contain a target of some robot in \poslit{}, or else $\ass$ does not satisfy $Q$.

In the path ensemble, each $R_i \in \neglit{}$ follows the subpath of $\pi_0$ from $s_i$ (through which $\pi_0$ passes) up to the gadget containing $f_i$, from which $R_i$ can reach its final position $f_i$ using the shortest path.
The order in which the robots in \neglit{} move is the right to left order of their start positions, which guarantees no collision with another robot located at its start position.
Since the robots in \neglit{} move before \poslit{}, the targets through which $\pi_0$ passes are unoccupied when the robots in \neglit{} move, guaranteeing no collisions at clause gadgets.
Next, $R_{0}$ moves using $\pi_0$, which passes through empty passages at this point.
Finally, each $R_i \in \poslit{}$ joins $\pi_0$ at the vertical passage to its right, from which point it continues similarly to \neglit{}.
The order of motion of the robots in \poslit{} is the right to left order of their targets, which guarantees no collisions in the clause gadgets.
Note that due to the intra-literal order property we also have no interferences among \poslit{} within variable gadgets.

For the other direction, let us assume that there is a monotone path ensemble for $\I$.
Let $\pi_0$ denote the path taken by $R_{0}$.
Without loss of generality, $\pi_0$ is weakly $x$-monotone.
Specifically, it passes through only one horizontal passage in each variable gadget.
Therefore, we define an assignment $\ass$ as follows: $x$ is assigned to be true if and only if $\pi_0$ goes through the bottom passage of $x$'s variable gadget, which corresponds to negative literals.
Let $C$ be a clause of $Q$ and let $f_j$ be a target in $C$'s clause gadget that is unoccupied during $R_{0}$'s motion, which must exist.
It is easy to verify that the literal corresponding to $R_j$ is true according to $\ass$.
Therefore, $C$ is satisfied. 
\end{proof}

The construction can be carried out in polynomial time, therefore by combining Lemma~\ref{lem:monotone-connection} and Lemma~\ref{thm:monotone-hardness} we obtain the following:
\begin{theorem}
MRMP-RA for weakly-monotone path ensembles is NP-hard.
\end{theorem}

\subparagraph*{Hardness of Approximation}
\newcommand{\satcost}{\mathop{\mathrm{SAT}}}
We now show that MRMP-RA is APX-hard, ruling out any polynomial time $(1+\eps)$-approximation algorithm.
We first go over some definitions.
For an MRMP-RA instance $\I$, we use $\cost^*(\I)$ to denote the cost of the optimal weakly-monotone path ensemble for $\I$.
For a 3SAT formula $Q$, let $\satcost(Q)$ denote the largest fraction of clauses in $Q$ that can be simultaneously satisfied.
We say that a revolving area $A$ is \emph{occupied} if it contains the robot whose start or target position lies in $A$.

To prove the hardness of approximation we present a gap-preserving reduction from MAX-3SAT(5), which is APX-hard~\cite{vazirani2001approximation}. The input to MAX-3SAT(5) is a 3SAT formula with 5 appearances for each variable and the goal is to find an assignment maximizing the number of satisfied clauses.
Let $Q$ be a MAX-3SAT(5) instance with $n$ variables and $m$ clauses and let $\I \coloneqq \I(Q)$ be the MRMP-RA instance resulting from the NP-Hardness reduction described above, which we slightly modify as follows.
Instead of the single pivot robot $R_0$ in $\I$, we now have $m$ pivot robots.
To this end, we modify the construction so that there is a horizontal passage that extends to the left of $s_0$ in $I$.
The passage is lengthened to accommodate $m$ start positions that lie on the same horizontal line, passing through $s_0$ in $I$.
Similarly, another such passage is created to the right of $f_0$ to accommodate $m$ target positions.
The left to right order of the start positions of the pivot robots is set to match the left to right order of the corresponding target positions.
Let $\I'$ denote the resulting MRMP-RA instance.

\begin{restatable}{lemma}{lemapx} 
\label{lem:apx-hardness1}
Let $Q$ be a 3SAT formula such that for any assignment to $Q$ there are at least $k$ unsatisfied clauses in $Q$.
Then $\cost^*(\I') > d(\I') + km$.
\end{restatable}

\begin{proof}
Let us examine $\pathens^*(\I')$, an optimal path ensemble for $\I'$.
We say that a robot $R_i$ has a \textit{bad event} during the execution of $\pathens^*(\I')$ when $R_i$ traverses an occupied revolving area. 
Note that each bad event results in $R_i$ having a path longer than 1+$d_i$, $d_i$ being the length of $R_i$'s shortest possible path.
We claim that each of the $m$ pivot robots has $k$ bad events, which suffices for proving the lemma.

Let us assume for a contradiction that one of the pivot robots, say $R_i$, has $q < k$ bad events.
We will show how to obtain an assignment for $Q$ where there are at most $q$ unsatisfied clauses.
Since $\pathens^*(\I')$ is optimal, $\pi_i$, the path taken by $R_i$, is weakly $x$-monotone.
We define an assignment $\ass$ as follows (the same way as in the second direction of the proof of Theorem~\ref{thm:monotone-hardness}): $x$ is assigned to be true if and only if $\pi_i$ goes through the bottom passage of $x$'s variable gadget.
In other words, $\ass$ sets a literal to be true if and only if the corresponding literal-robot's starting position does not lie on $\pi_i$.
Let us examine $\pi_i$ right before it is $R_i$'s turn to move.
Let $\R$ denote the set of robots that are intersected by $\pi_i$ and are located at variable gadgets at this point in time.
We can assume without any loss of generality that $\R$ is empty.
If it is not, then let us examine the path ensemble $\pathens$ where the robots in $\R$ move to their targets before $R_i$'s turn.
The number of bad events for $R_i$ can only decrease in $\pathens$.
This holds because by having some $R_j \in \R$ move before $R_i$ we eliminate a bad event (for $R_i$) in $R_j$'s variable gadget and possibly introduce a bad event in $R_j$'s clause gadget.

Since there are $q$ bad events for $R_i$, there are at most $q$ clause gadgets where such an event occurs.
Therefore, to get a contradiction it suffices to show that all other clause gadgets correspond to clauses that are satisfied by $\ass$.
Let $C$ be such a clause, i.e., in the corresponding clause gadget $\pi_i$ passes through some empty revolving area $A_{f_j}$.
Since $\pi_i$ does not pass through any occupied revolving areas in the variable gadgets, the corresponding start position $s_j$ must not lie on $\pi_i$.
Therefore, $r_j$ corresponds to a literal that is true by $\ass$, and so $C$ is satisfied.
\end{proof}

We now make $d(\I')$ explicit using an upper bound for an arbitrary $d_i$.
First, we bound the length of each vertical segment in the corresponding path $\pi_i$ by 10, which provides sufficient distance for our gadgets.
Since each variable appears in $Q$ five times, we bound the horizontal length of an variable gadget by $4\cdot2+3=11$ (i.e., there at most 4 revolving areas on a horizontal passage and some additional length).
Therefore, the path length through any gadget is $O(1)$.
Hence, we have $d_i = O(m)$ and the number of robots is also $O(m)$ (we have $m = 5n/3$).
Therefore, we can set $d(\I')=cm^2$ for some sufficiently large constant $c$ (we can easily lengthen paths in $\I'$ if that is needed for the bound).
\ 
We can now combine the latter equality with Lemma~\ref{lem:apx-hardness1} and the NP-Hardness reduction.
Let us define $f(Q) \coloneqq d(\I')$.
\ 
\begin{theorem}
There is a polynomial time reduction that transforms an instance $Q$ of MAX-3SAT(5) with $m$ clauses to an MRMP-RA instance $\I'$ such that $\satcost(Q) = 1 \Rightarrow \cost^*(\I')=f(Q) \leq cm^2$ for some constant $c >0$ and otherwise 
$\satcost(Q) < \alpha \Rightarrow \cost^*(\I') > f(Q) + (1-\alpha)m \cdot m = 
\left( 1 + \frac{1-\alpha}{c}\right ) f(Q),\  \text{for all } 0 < \alpha < 1$.
\end{theorem}

\section{Algorithm}
\label{sec:alg}
Let $\I = (\W, \starts, \finals, \A)$ be an instance of MRMP-RA. Let $n$ be the number of robots and $m$ be the complexity of the environment $W$.
We describe an $(n(n+m)\log m)$ algorithm for computing a weakly-monotone path ensemble $\tilde{\pathens} \coloneqq \tilde{\pathens}(\I)$ for $R_1, \dots, R_n$ such that $\cost(\tilde{\pathens}) = O(1) \cdot \cost^*(\I)$. 
We remark that $\tilde{\pathens}$ is weakly-monotone but $\pathens^*(\I)$ need not be, i.e. $\tilde{\pathens}$ is an $O(1)$-approximation of any feasible motion plan.
We parameterize the paths in $\tilde{\pathens}$ over the common interval $J = [0, n]$.
We need a few definitions and concepts related to revolving areas. 
For any $z \in \starts \cup \finals$, let $c_z$ denote the center of the revolving area $A_z$, and let $C_z$ (resp. $B_z$) be the disc of radius $1$ (resp. 3) centered at $c_z$, i.e., $C_z \subset A_z \subset B_z$. 
If $x \not \in B_z$ then $D_x \cap A_z = \emptyset$. 
We refer to $C_z$ and $B_z$ as the \textit{core} and \textit{buffer}, respectively, of revolving area $A_z$. 
See Figure~\ref{fig:simple-revolving-area}.

\subparagraph{Overview of the Algorithm.}
The algorithm consists of three stages. We note that Stage (I) and (II) are used in~\cite{SolomonHalperin2018}.  
However, Stage (III) differs significantly from previous work in order to ensure the total cost of paths is within an $O(1)$ factor of that of the optimal motion plan. 
We describe all stages for completeness. 

\begin{figure}[t]
    \centering
    \includegraphics[width=0.7\textwidth]{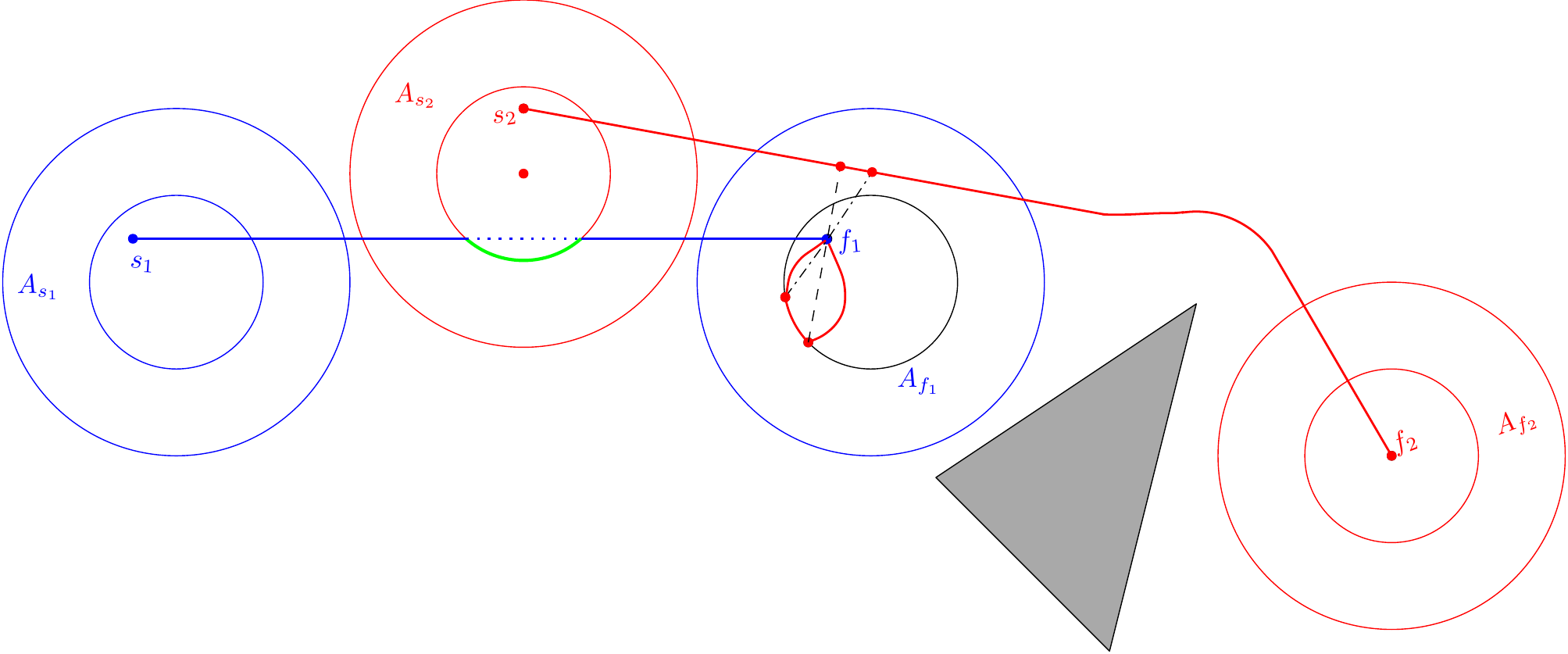}
    \caption{Path $\gamma_1$ is shown in blue from $s_1$ to $f_1$. Assume $R_1$ is active before $R_2$. In $\bar{\gamma}_1$, the dotted portion of the path is replaced with the green arc along $\partial C_{s_2}$.
   Path $\pi_2$ is shown in red from $s_2$ to $f_2$. $R_1$ must follow the red retraction during the movement of $R_2$ in $B_{f_2}$.}
    \label{fig:overall-paths}
\end{figure}

\begin{description}
    \item[I.] We compute the free space $\F$ (with respect to one robot) using the algorithm of Ó'Dúnlaing and Yap~\cite{DBLP:journals/dcg/Yap87,o1985retraction}. 
    If $s_i$ and $f_i$, for some $i \in [n] \coloneqq \{ 1, 2, \dots, n \}$, do not lie in the same connected component of $\F$, then a feasible path does not exist for $R_i$ from $s_i$ to $f_i$.
    Therefore, we stop and return that no feasible motion plan exists from $\starts$ to $\finals$. 
    Next, for each $i$, we compute a shortest path $\gamma_i$ from $s_i$ to $f_i$, ignoring other robots using the algorithm of Chen and Wang~\cite{chen2015computing}.
    Let $\shrtpath = \{ \gamma_1, \dots, \gamma_n \}$ be the path ensemble computed by the algorithm. 
    
    Although $\shrtpath$ does not intersect $\O$, it may not be feasible since two robots may collide during the motion. 
    The next two steps deform $\shrtpath$ to convert it into a feasible motion plan. 
    We take an arbitrary permutation $\sigma$ of $[n]$. 
    Without loss of generality assume $\sigma = \seq{1, 2, \dots, n}$.
    We say that $R_i$ is \emph{active} during the subinterval $[i-1, i]$ of $J \coloneqq [0, n]$, during which it moves from $s_i$ to $f_i$. During $[0, i-1]$ (resp. $[i, n]$) $R_i$ only moves within the revolving area $A_{s_i}$ (resp. $A_{f_i}$).
    \item[II.] For each $i$, we first modify $\gamma_i$, as described below in Section~\ref{subsec:path-deformation}, so that it does not intersect the interior of the core $C_j$ of any revolving area $A_j$ that is occupied by a robot $R_j$, for $j \neq i$; see Figure~\ref{fig:overall-paths}.
    Let $\overline{\gamma}_i$ be the deformed path. 
    Abusing the notation a little, let $\overline{\gamma}_i: [i-1, i] \rightarrow \F$ denote a uniform parameterization of the path $\overline{\gamma}_i$, i.e. $R_i$ moves with a fixed speed during $[i-1, i]$ from $s_i$ to $f_i$ along $\overline{\gamma}_i$. 
    We extend $\bar{\gamma}_i$ to the interval $[0, n]$ by setting $\bar{\gamma}_i(t) = s_i$ for $t \in [0, i-1]$ and $\bar{\gamma}_i(t) = f_i$ for $t \in [i, n]$. 
    Set $\bar{\Gamma} = \{ \bar{\gamma}_1, \dots, \bar{\gamma}_n\}$. 
    \item[III.] Next, for each distinct pair $i, j \in [n]$, we construct a \emph{retraction map} $\rho_{ij}: \F \rightarrow \F$ that specifies the position of $R_j$ for a given position of $R_i$ during the interval $[i-1, i]$ when $R_i$ is active so that $R_i$ and $R_j$ do not collide as $R_i$ moves along $\overline{\gamma}_i$.
    The retraction map ensures that $R_j$ stays within the revolving area $A_{s_j}$ (resp. $A_{f_j}$) for $j < i$ (resp. $j > i$), and it does not collide with any $R_k$ for $k \neq i, j$, as well. See Figure~\ref{fig:overall-paths}.
    Using this retraction map, we construct the path $\pi_j: J \rightarrow \F$ as follows: 
    $\pi_j(t) = \begin{cases} \rho_{ij}(\overline{\gamma}_i(t)) \ &\text{for } t \in [i-1, i] \text{ and } i \neq j, \\ 
    \overline{\gamma}_j(t)  &\text{for } t \in [j-1, j]. 
    \end{cases}$
    
    We prove below that each $\pi_j$ is a continuous path. In Section~\ref{sec:analysis}, we prove that $\pathens = \{ \pi_1, \dots, \pi_n\}$ is a feasible path ensemble with $\cost(\pathens) = O(1)\cdot \optc$.
\end{description}

\subsection{Modifying  path \texorpdfstring{$\bm{\gamma}_i$}{TEXT}}
\label{subsec:path-deformation}
Fix an $i \in [n]$. 
For $j < i$, let $z_j = f_j$ and for $j > i$, let $z_j = s_j$. 
Set $Z = \{z_j : 1 \leq j \neq i \leq n \}$.
This step modifies $\gamma_i$ to ensure that the path of $R_i$ does not enter the core $C_{z}$ of any $z \in Z$.

Fix a $z \in Z$. If $\gamma_i \cap C_{z} = \emptyset$, then $R_j$ does not affect $\gamma_i$. 
If $\gamma_i \cap C_z \neq \emptyset$, then we modify $\gamma_i$ as follows: let $p_z$, $q_z$ be the first and last intersection points of $\gamma_i$ and $C_{z}$ along $\gamma_i$, respectively. 
Let $Q_z$ be the shorter arc of $\partial C_{z}$, the boundary of the core $C_{z}$, between $p_z$ and $q_z$. 
We replace $\gamma_i[p_z, q_z]$ with $Q_z$. 
We repeat this step for all $z \in Z$. 
Let $\overline{\gamma}_i$ be the resulting path from $s_i$ to $f_i$; $\overline{\gamma}_i$ does not intersect $\interior(C_z)$ for any $z \in Z$. 
Note that $C_{z_j}$'s are pairwise disjoint, and that $\gamma_i$ is a shortest path from $s_i$ to $f_i$ in $\F$, therefore $\gamma_i[p_z, q_z]$ and $\gamma_i[p_{z'}, q_{z'}]$, for any pair $z, z' \in Z$, are disjoint. 
We can thus process $Z$ in an arbitrary order and the resulting path does not depend on the ordering. 
Furthermore $C_z \subseteq \F$ (since $A_z \subseteq \W$), so $\bar{\gamma}_i \subset \F$ for all $i$. 

\subsection{Retracting a robot \texorpdfstring{$\boldmath{R_j}$}{}}

\begin{figure}
    \captionsetup[subfigure]{justification=centering}
     \centering
     \begin{subfigure}[b]{0.4\textwidth}
         \centering
         \includegraphics[width=\textwidth]{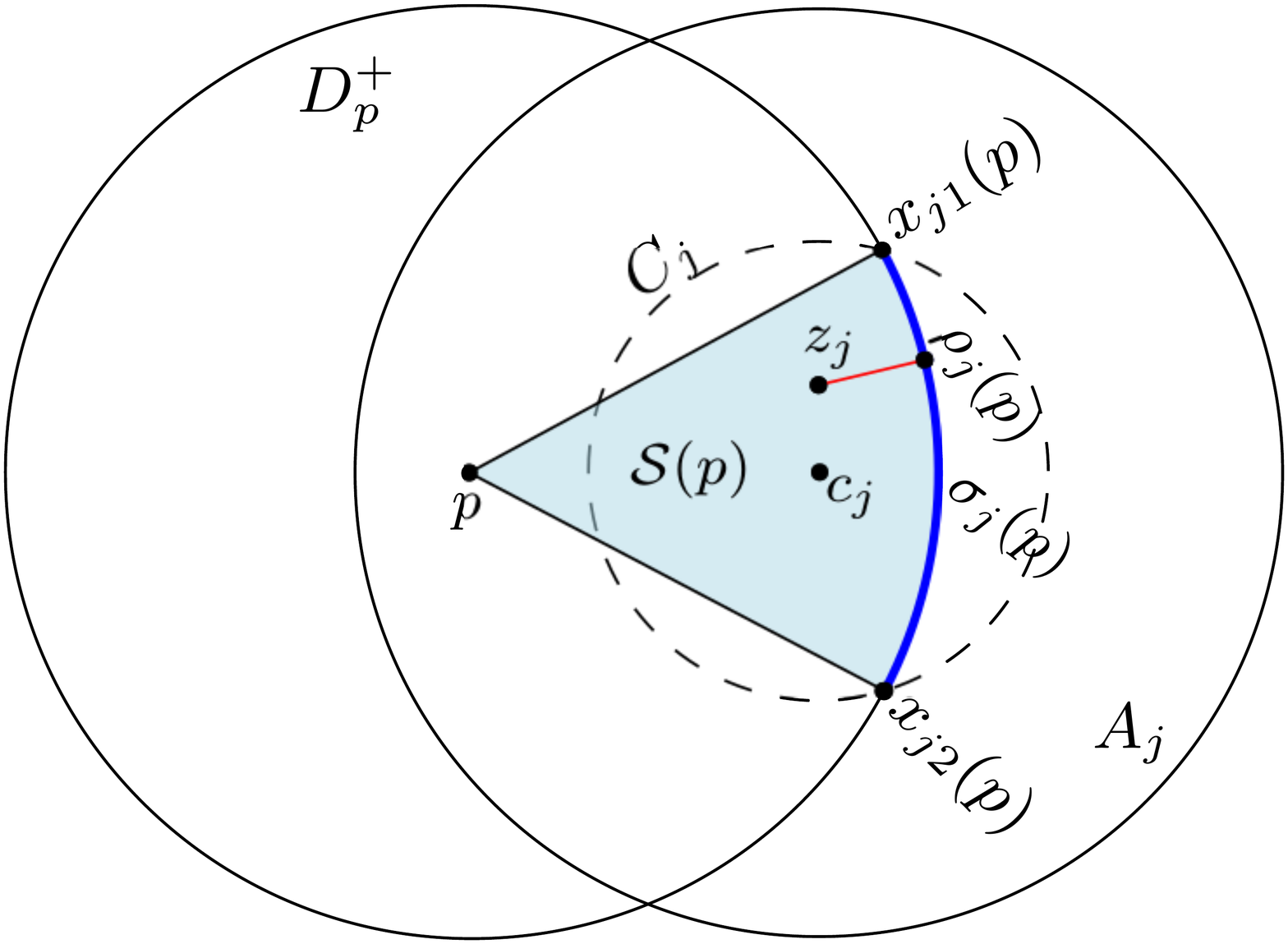}
         \caption{}
     \end{subfigure}
     \hfill
     \begin{subfigure}[b]{0.4\textwidth}
         \centering
         \includegraphics[width=\textwidth]{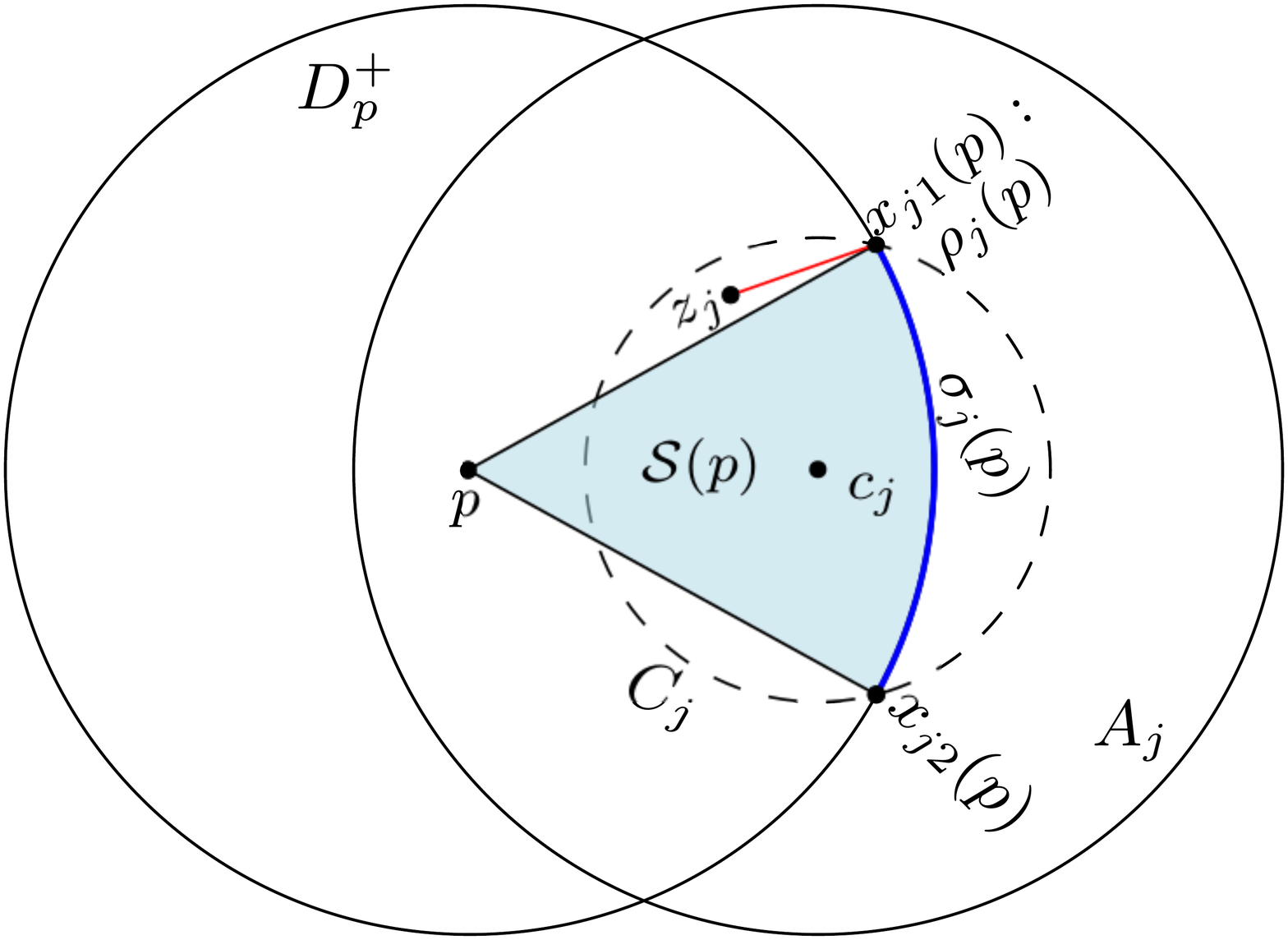}
         \caption{}
     \end{subfigure}
\caption{Retraction Map. (a) A sector type retraction (when $z_j$ lies in $S(p)$); (b) an intersection type retraction, $z_j$ lies outside $S(p)$, the retraction point is $x_{j1}(p)$.}
\label{fig:retraction}
\end{figure}

We now describe the retraction motion of $R_j$ when $R_i$ is active, so that they do not collide. 
Note that for all $t \in [i-1, i]$, $\bar{\gamma}_j(t) = z_j$, i.e., before applying the retraction $R_j$ is at $z_j$ when $R_i$ is active. 
We define the \emph{retraction function} $\rho_{ij}: \reals^2 \rightarrow \reals^2$ that specifies the motion of $R_j$ within $A_{j}$ during $[i-1, i]$. 
Since $i$ is fixed, for simplicity we use $\rho_j$ to denote $\rho_{ij}$, and we use $C_j$ (resp. $A_j, B_j$) for disc $C_{z_j}$ (resp. $A_{z_j}, B_{z_j}$).
If the center of $R_i$ is at distance at least $2$ from $z_j$, then $R_i$ does not intersect $D_{j}$, so there is no need to move $R_j$ from $z_j$. 
Therefore we set $\rho_{j}(p) = z_j$ for all $p \in \pi_i$ such that $\norm{p - z_j} \geq 2$. 
On the other hand, $\overline{\gamma}_i$ does not intersect the interior of $C_{j}$ so $\rho_j(p)$ is undefined for $p \in \interior(C_{j})$. 
We thus focus on the case when $\norm{p - z_j} \leq 2$, in which case $p$ lies in the buffer disc $B_{j}$, and $p \not \in \interior(C_{j})$, i.e., $p \in B_{j} \setminus \interior(C_{j})$. 

Let $D^+(p)$ be the disc of radius $2$ centered at $p$. 
Note that for a point $q \in \reals^2$, $\interior(D_p) \cap\  \interior(D_q) \neq \emptyset$ if and only if $q \in D^+(p)$. 
Intuitively, we move the center of $R_j$ from $z_j$ (within $C_{j}$) as little as possible so that $R_j$ does not collide with $R_i(p)$.
Formally, we define $\rho_j$ as: $\rho_j(p) = \argmin_{q \in C_{j} \setminus D^+_p} \norm{q - z_j}$ if $p \not \in \interior(C_{j})$, and undefined otherwise.  

In the remainder of the discussion, we assume $\norm{z_j - p} \leq 2$ and $p \not \in C_j$, so $p \in B_j \setminus C_j$. 
Therefore, $\rho_j(p)$ exists and additionally $\rho_j(p)$ is unique. 
We now discuss the two possible types of retraction. Refer to Figure~\ref{fig:retraction} throughout this paragraph.
Note that $\partial C_j$ and $\partial D^+_p$ intersect at exactly two points since $p \in B_j \setminus C_j$, say, $x_{j1}(p), x_{j2}(p) $.
Let $\sigma_j(p)$ be the smaller of the two arcs of $\partial D^+_p$ induced by $x_{j1}(p)$ and $x_{j2}(p)$, and let $S(p) = \text{conv}(\sigma_j(p) \cup \{ p\}) \subseteq D^+_p$ be the \emph{sector} of $D^+(p)$ induced by $x_{j1}(p)$, $x_{j2}(p)$.
Observe that the retraction point $\rho_j(p)$ lies on $\sigma_j(p)$. 
If $z_j \in S(p)$, then $\rho_j(p)$ is the intersection point of the ray $\overrightarrow{pz_j}$ with $\partial D^+_p$, as the closest point in $C_{j} \setminus D^+_p$ from $z_j$ is on the straight line from $z_j$ to $D^+_p$.
Since $z_j$ lies inside $S(p)$, $\rho_j(p) \in \partial S(p)$.
If $z_j \not \in S(p)$, the retraction point is
$\argmin_{q \in \{x_{j1}(p), x_{j2}(p) \}} \norm{ q - z_j }$, i.e., the closest point to $z_j$ in $C_{j} \setminus D^+_p$ is an endpoint of $\sigma_j(p)$.
Note that our retraction ensures that $R_j$ will be centered back at $z_j$ after robot $R_i$ moves away.


In the remainder of the paper, if $\rho_j(p) \in \{ x_{j1}(p), x_{j2}(p) \}$ we say the that the retraction is of \emph{intersection} type, otherwise we say that the retraction is of \emph{sector} type.
Since $\rho_j(p) \in C_j$ for all $p \not \in C_j$ and none of the $\bar{\gamma}_i$'s enter $C_j$, the retraction path $\pi_j$ of $R_j$ lies in $\F$. 
We conclude this section with the following lemma, which follows from the fact that $\rho_j$ is a continuous function, $\rho_j(p) = z_j$ for all $p$ such that $\norm{p - z_j} \geq 2$, and $\norm{z_j - s_i}, \norm{z_j - f_i} \geq 2$.  

\begin{lemma}
For any $1 \leq i \leq n$, $\pi_i$ is a continuous path from $s_i$ to $f_i$. 
\end{lemma}


\section{Correctness and Analysis of the Algorithm}
\label{sec:analysis}
We first prove that $\Pi$ is feasible (Section~\ref{sec:feasibility}), then we bound $\cost(\Pi)$ (Section~\ref{sec:cost}), and finally analyze the running time in Section~\ref{sec:runningtime}. 
We begin by summarizing a few relevant properties of revolving areas (see Figure~\ref{fig:revolving-area-lemma}), which are straightforward to prove. 

\begin{lemma}
\label{lem:ra-ppts}
\begin{inparaenum}[(i)]
Let $x, y \in \starts \cup \finals$ such that $x \neq y$:
\item  $x \in C_x$, that is, each start or final position lies inside the core of $A_x$;
\item $\norm{c_x - c_y} \geq 2$, i.e., $\interior\ C_x \cap \interior\ C_y = \emptyset$;
\item for any $p \in C_x$, $\norm{p - y} \geq 2$, i.e., $\interior\ D_p \cap \interior\ D_y = \emptyset$; 
\item $\norm{x - c_y} \geq 3$, i.e., each start/final position lies outside the buffer of any other start/final position. 
\end{inparaenum} 
\end{lemma}

\begin{figure*}[h]
\centering
\includegraphics[width=0.35\textwidth]{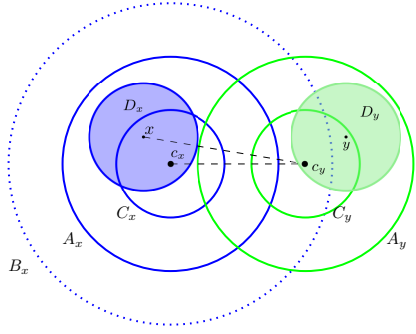}
\caption{Illustration of Lemma~\ref{lem:ra-ppts}. $D_x$ and $D_y$ are two robots located in their respective revolving areas $A_x, A_y$. The distance between $c_x$ and $c_y$ is at least $2$; $y$ lies outside the buffer of $x$, i.e., $y \not \in B_x$.}
\label{fig:revolving-area-lemma}
\end{figure*}

\subsection{Feasibility}
\label{sec:feasibility}
In this section, we show that path ensemble $\pathens$ is feasible. 
Recall that stage (I) of the algorithm reports that there is no feasible solution if any $s_i \in \starts$ and $f_i \in \finals$ do not lie in the same connected component. 
So assume that Stage I computes a feasible path $\gamma_i$ for each $R_i$. 
Stages II and III modify these paths so that they remain in $\F$. 
Hence, we only need to show that no two robots collide with each other during the motion, i.e., for any $1 \leq i \neq j \leq n$ and for any $t \in J$, $\interior\ D_{\pi_i(t)} \cap \interior\ D_{\pi_j(t)} = \emptyset$.
We fix some $i\in [n]$ and the corresponding active interval $T_i \coloneqq [i-1, i] \subseteq J$ and prove the feasibility of $\Pi$ during this interval. 
Note that $R_i$ is the only active robot in $T_i$ and other robots stay in their revolving areas. 
By the definition of retraction, for any $t \in T_i$, and for any $j \neq i$, $\norm{\pi_i(t) - \rho_{ij}(\pi_i(t))} \geq 2$, so $R_i$ does not collide with $R_j$ during interval $T_i$.
Thus, we only need to show that for any pair $j, k \neq i$, $R_j$ and $R_k$ do not collide while $R_j$ moves along its retraction path. \short{The following lemmas are proved in Appendix~\ref{sec:appendix}.} Since Lemma~\ref{lem:no-collision} holds for any interval $T_i$, we obtain the final statement of feasibility. 

\eat{\begin{figure}[h]
    \captionsetup[subfigure]{justification=centering}
     \centering
     \begin{subfigure}[b]{0.45\textwidth}
         \centering
         \includegraphics[width=\textwidth]{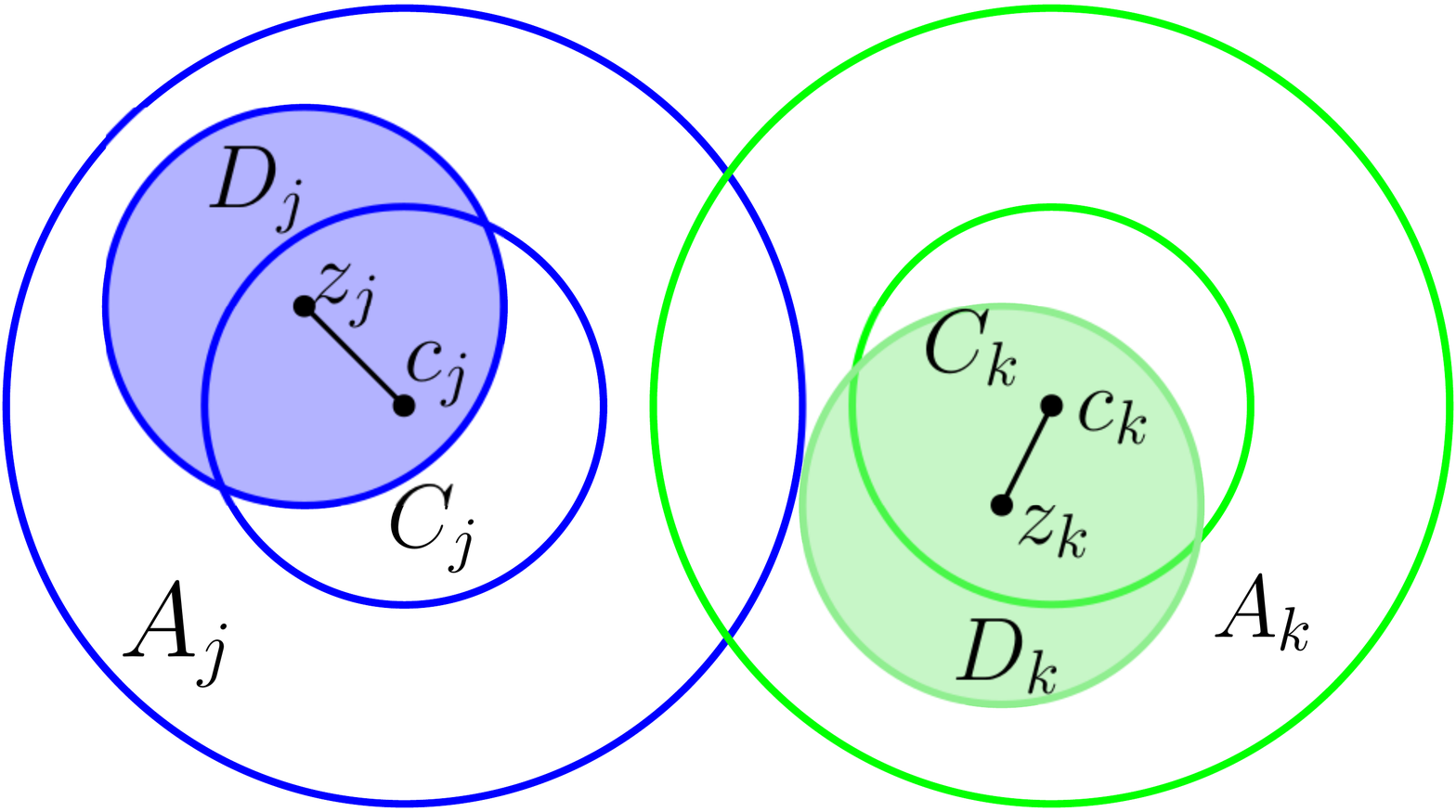}
         \caption{}
     \end{subfigure}
     \hfill
     \begin{subfigure}[b]{0.45\textwidth}
         \centering
         \includegraphics[width=\textwidth]{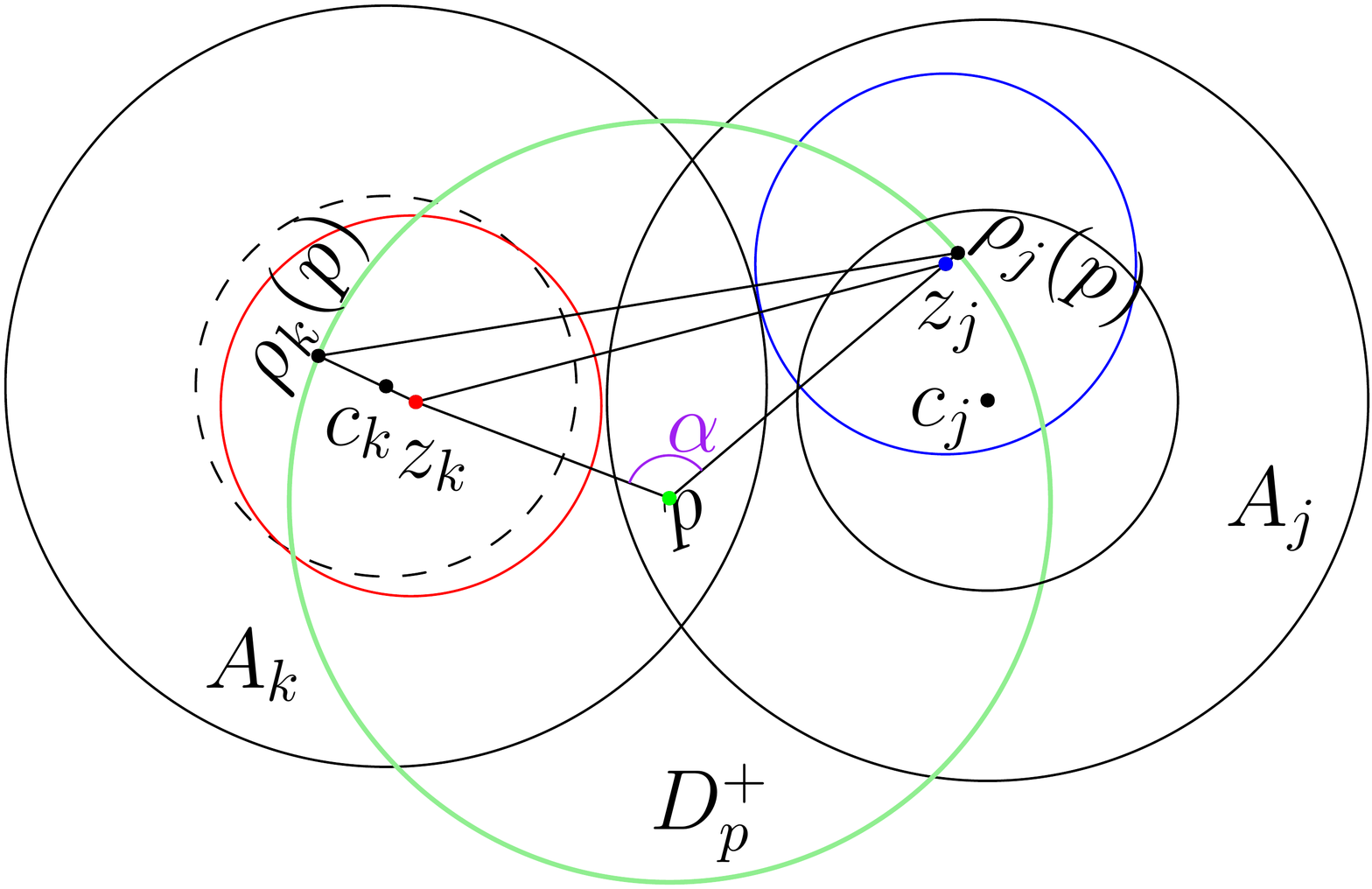}
         \caption{}
     \end{subfigure}
  \caption{ (a) Illustration of Lemma~\ref{lem:dist-segments}. (b) Case 2 of Lemma~\ref{lem:no-collision}, The angle $ \alpha \geq \pi/3$. \tzvika{Perhaps separate into two figures?}} 
    \label{fig:segment}
\end{figure}}

\begin{restatable}{lemma}{segment}
\label{lem:dist-segments}
For any $j \neq k$ and $z_j, z_k \in \starts \cup \finals$, the minimum distance between the line segments $c_j z_j$ and $c_k z_k$ is at least $2$, i.e., 
$\min_{\substack{y_j \in c_jz_j,\\ y_k \in c_k z_k}} \norm{y_j - y_k} \geq 2.$ 
\end{restatable}


\eat{\begin{proof}
Let $y_j, y_k$ be the closest pair of points on the segments $c_jz_j$ and $c_kz_k$. 
Note that $z_j c_j$ and $z_k c_k$ are disjoint since $z_j c_j \in C_j$ and $z_k c_k \in C_k$ and these cores do not intersect (cf Lemma~\ref{lem:ra-ppts}). 
This implies that either $y_j$ or $y_k$ must be an endpoint of the respective segment. 
\ 
Assume without loss of generality that $y_j$ is an endpoint of $z_j c_j$.
By Lemma~\ref{lem:ra-ppts}, $\norm{c_j - z_k}, \norm{c_k - z_j} \geq 3$. 
Let $y'_k$ be the endpoint of $c_z z_k$ at distance $3$ from $y_j$ ($y'_k = z_k$ if $y_j = c_j$ and $y'_k = c_k$ otherwise). 
Since $z_k \in C_k$, $\norm{y_k - y'_k} \leq 1$, 
Then $\norm{y_j - y_k} \geq \norm{y_j - y'_k} - \norm{y'_k - y_k} \geq 3 - 1 = 2$. 
\end{proof}}

\begin{restatable}{lemma}{collisionfree}
\label{lem:no-collision}
For any $j \neq k$, $R_j$ and $R_k$ do not collide during the interval $T$.
\end{restatable}

\eat{
\begin{proof} 
In view of the above discussion, we assume $j, k \neq i$. 
The claim is equivalent to showing that $\norm{\rho_j(\pi_i(t)) - \rho_k(\pi_i(t))} \geq 2$ for every $t \in T$. Let $p = \pi_i(t)$.
There are two cases: 
\subparagraph*{Case 1: $\rho_j(p) = z_j$ or $\rho_k(p) = z_k$.}
Without loss of generality, assume that $\rho_j(p) = z_j$. 
By construction, $\rho_k(p) \in C_k$, therefore by Lemma~\ref{lem:ra-ppts}(iii), $\norm{\rho_j(p) - \rho_k(p)} = \norm{z_j - \rho_k(p)} \geq 2$. 
\subparagraph*{Case 2: $\rho_j(p) \neq z_j$ and $\rho_k(p) \neq z_k$.}
Recall $D^+_p$ is the disc of radius 2 centered at $p$, and $x_{j,1}, x_{j,2}$ are the intersection points of the core of $A_j$ and $\delta D_p^+$. 
In this case, $z_j, z_k \in D^+_p$. 
We consider the triangle formed by the retraction points $\rho_k(p)$, $\rho_j(p)$ and $p$. We show that $\angle \rho_k(p) p \rho_j(p) \geq \pi/3$. 
We will first define a point $f_j(p)$ based on the current retraction type of $R_j$. 
%
If the retraction of $R_j$ is type sector, then $z_j$ lies within the sector $S(p)$, let $f_j(p) = z_j$.
Otherwise, the retraction is of type intersection and without loss of generality we assume $\rho_j(p)$ is $x_{j, 1}(p)$. In this case, consider the segments $px_{j,1}(p)$ and $z_j c_j$. These two segments must intersect, as $c_j \in \S(p)$ and $z_j \not \in \S(p)$. 
We let $f_j(p)$ be the intersection point of segments.
Note that in either case, $f_j(p)$ lies on segment $p \rho_j(p)$.
We analogously  define $f_k(p)$. 
See Figure~\ref{fig:segment} for an example where $f_j(p) = z_j$ and $f_j(p) = z_k$.
By definition, $f_j(p) \in z_j c_j$ and $f_k(p) \in z_k c_k$ and Lemma~\ref{lem:dist-segments} implies that $\norm{f_j(p) - f_k(p)} \geq 2$.
Additionally, $f_j(p) \in D^+_p$, so $\norm{p - f_j(p)} \leq 2$ (similarly  $\norm{p - f_k(p)} \leq 2$). 
Let $\alpha$ be the angle $\angle f_k(p) p f_j(p)$.
Since $\norm{p - f_j(p)}, \norm{p - f_k(p)} \leq 2$, and $\norm{f_j(p) - f_k(p)} \geq 2$, $\alpha \geq \pi/3$. 
Now consider the triangle formed by the retraction points and $p$.
By construction, $\angle f_k(p) p f_j(p) = \angle \rho_k(p) p \rho_j(p)$. 
The distance between $p$ and each retraction point is $2$: $\abs{p\rho_j(p)} = \abs{p \rho_k(p)} = 2$. 
This implies the other two angles in the triangle are equal ($\angle p \rho_k(p) \rho_j(p) = \angle p \rho_j(p) \rho_k(p)$).
Since $\angle \rho_k(p) p \rho_j (p) = \alpha \geq \pi/3$, $\rho_j(p) \rho_k(p)$ is the longest edge of the triangle $\triangle p \rho_j(p) \rho_k(p)$.
The other two sides have length $2$, so $\norm{\rho_j(p) - \rho_k(p)} \geq 2$, as desired.
\end{proof}}

\begin{corollary}
The path ensemble $\Pi$ returned by the algorithm is feasible. 
\end{corollary}

\subsection{Cost of path ensemble}
\label{sec:cost}
We now analyze the cost of the path ensemble $\pathens$ the algorithm returns. 
The algorithm starts by computing $\Gamma$, the shortest paths  of all robots in $\F$ while ignoring other robots. 
Clearly, we have $\cost(\shrtpath) \leq \cost^*(\I)$. 
We show that $\cost(\pathens) = O(\cost(\shrtpath))$. 
\short{In Appendix~\ref{section:gammabar}, we prove $\cost(\bar{\Gamma}) \leq 2 \cost(\Gamma)$,
so we focus on bounding the length of retraction paths of non-active robots, which is one of the main technical contributions of the paper.}
\eat{Stage (II) of the algorithm deforms $\shrtpath$ to $\overline{\Gamma}$. 
Path $\gamma_i \neq \overline{\gamma}_i$ only if  $\gamma_i \cap\  \interior\ C_j \neq \emptyset$ for some $j \neq i$, otherwise $\ell(\gamma_i) = \ell(\overline{\gamma}_i)$.  
Suppose  $\gamma_i \cap C_j \neq \emptyset$ for some $j \neq i$. Then in $\overline{\gamma}_i$, $\gamma_i \cap C_j$ is replaced with the shorter arc $\sigma$ of $\partial C_j$, determined by the first and last endpoints, say $p$ and $q$, of $\gamma_i \cap \partial C_j$.
Therefore, $\ell(\sigma) \leq 2 \sin^{-1}\Paren{ \frac{\norm{p - q}}{2}} \leq 2 \ell(C_j \cap \gamma_i)$. 
Hence, $\ell(\bar{\gamma}_i) \leq 2 \ell(\gamma_i)$ and we obtain:
$\cost(\bar{\Gamma}) \leq 2 \cost(\Gamma)$.}

\eat{We now focus on bounding the length of retraction paths of non-active robots, which is one of the main technical contributions of the paper.}
Let $\pi_{ji} = \pi_j[i-1, i]$, and let $\Delta_{ij} = \{t \in [i-1, i] : \norm{\pi_i(t) - z_j} \leq 2 \}$, i.e., $\pi_{ji}$ is the retraction of $R_j$ due to the motion of $R_i$ and $\pi_i[\Delta_{ij}]$ is the part of $\pi_i$ that causes the retraction motion of $R_j$. 
Refer to Figure~\ref{fig:overall-paths}. 
We show that $\ell(\pi_{ji}) = O(\ell(\pi_i[\Delta_{ij}]))$ (cf Corollary~\ref{cor:all-retraction}) and charge $\pi_{ji}$ to $\pi_i[\Delta_{ij}]$. 
We bound $\ell(\pi_{ji})$ by splitting into two scenarios. 
Roughly speaking, if $\pi_i$ does not penetrate the buffer $B_j$ too deeply, we use a Lipschitz condition on the retraction map to show $\ell(\pi_{ji}) = O(\ell(\pi_i[\Delta_{ij}]))$. 
More concretely, for $z \in \starts \cup \finals$, let $W_z$ be the disk of radius $3/2$ centered at $z$. 
We prove a Lipschitz condition when the active robot lies outside $W_j$ (cf Corollary~\ref{cor:outsidewj}).
On the other hand, if $\pi_i$ travels into $W_j$ then the Lipschitz condition may not hold, but we argue that $\ell(\pi_i[\Delta_{ij}]) = \Omega(1)$ and that $\ell(\pi_{ji}) = O(1)$ (cf Lemma~\ref{lem:inside-wj}). 
Finally, using a packing argument, we show that each ``point" of $\pi_{i}$ is only charged $O(1)$ times, and thus $\cost(\pathens) = O(\cost(\bar{\shrtpath})) = O(\cost(\shrtpath))$. 

\eat{\subparagraph{Cost of $\overline{\Gamma}$.} 
\label{section:gammabar}
Stage (II) of the algorithm deforms $\shrtpath$ to $\overline{\Gamma}$. 
Path $\gamma_i \neq \overline{\gamma}_i$ only if  $\gamma_i \cap\  \interior\ C_j \neq \emptyset$ for some $j \neq i$, otherwise $\ell(\gamma_i) = \ell(\overline{\gamma}_i)$.  
Suppose  $\gamma_i \cap C_j \neq \emptyset$ for some $j \neq i$. Then in $\overline{\gamma}_i$, $\gamma_i \cap C_j$ is replaced with the shorter arc $\sigma$ of $\partial C_z$, determined by the first and last endpoints, say $p$ and $q$, of $\gamma_i \cap \partial C_j$.
Therefore, $\ell(\sigma) \leq 2 \sin^{-1}\Paren{ \frac{\norm{p - q}}{2}} \leq 2 \ell(C_j \cap \gamma_i)$. 
Hence, $\ell(\bar{\gamma}_i) \leq 2 \ell(\gamma_i)$ and we obtain:
\ 
\begin{lemma}
\label{lem:gammabar}
$\cost(\bar{\Gamma}) \leq 2 \cost(\Gamma)$.
\end{lemma}}
\subparagraph{Retraction of $\mathbf{R_j}$ outside $\mathbf{W_j}$.} 
\label{section:outsidewj}
As in Section~\ref{sec:feasibility}, we fix an interval $[i-1, i]$ for some $i \in [n] \setminus \{j \}$. 
Let $\Delta^o_j = \{t \in [i-1, i] : \norm{\pi_i(t) - z_j} \leq 2 \text{ and } \pi_i(t) \not \in W_j \}$. 
That is, $\Delta^o_j$ is the interval(s) of time in which the path of robot $R_i$ forces the retraction of robot $R_j$ while the center of $R_i$ lies outside $W_i$. 
Let $\Phi_{ij}$ be the restriction of path $\pi_i$ of robot $R_i$ during the interval $\Delta^o_j$, i.e. $\Phi_{ij}(t) = \pi_i(t)$ for $t \in \Delta^o_j$. 
Let $\Psi_{ji}: \Delta^o_j \subset \Delta_{ij} \rightarrow C_j$ be the retraction of $R_j$ during $\Delta^o_j$, i.e., $\Psi_{ji}(t) = \rho_{ij}(\pi_i(t))$ for $t \in \Delta^o_j$. 
We show that $\ell(\Psi_{ji}) = O(\ell(\Phi_{ij}))$ by proving a Lipschitz condition on $\ell(\Psi_{ji})$.

We will divide $\Phi_{ij}$ into subpaths, referred to as pathlets, so that there is only one type of retraction point associated with the subpath.  
We call a time instance $t \in \Delta^o_j$ an \emph{event} if $t$ is either an endpoint of a connected component of $\Delta^o_j$ (i.e., $\norm{\pi_i(t) - c_j} = 3/2$ or $\norm{\pi_i(t) - z_j} = 2$) or $z_j \in \partial \S_j(\pi_i(t))$, (i.e., the type of retraction point $\rho_j(\pi_i(t))$ changes at time $t$). 
Let $t_0 < t_1 < \dots < t_k$ be the event points.
We divide $\Phi_{ij}$ and $\Psi_{ji}$ into \emph{pathlets} at these events, i.e., $\Phi_{ij} = \varphi_1 \circ \varphi_2 \circ \cdots \circ \varphi_g$ and $\Psi_{ji} = \psi_1 \circ \psi_2 \circ \cdots \circ \psi_g$ where $\varphi_k = \pi_i[t_{k-1}, t_k]$ and $\psi_k = \rho_j(\varphi_k) = \pi_j[t_{k-1}, t_k]$.
We prove the Lipschitz condition for each pathlet. 
All points on $\rho_j(\varphi_k)$ have the same type of of retraction by construction of $\Phi_{ji}$. 
We call $\varphi_k$ a \emph{sector-type} (\emph{intersection-type}) pathlet if all points have sector (resp. intersection) type retraction.  
 


\begin{figure}
    \captionsetup[subfigure]{justification=centering}
     \centering
     \begin{subfigure}[b]{0.45\textwidth}
         \centering
         \includegraphics[width=\textwidth]{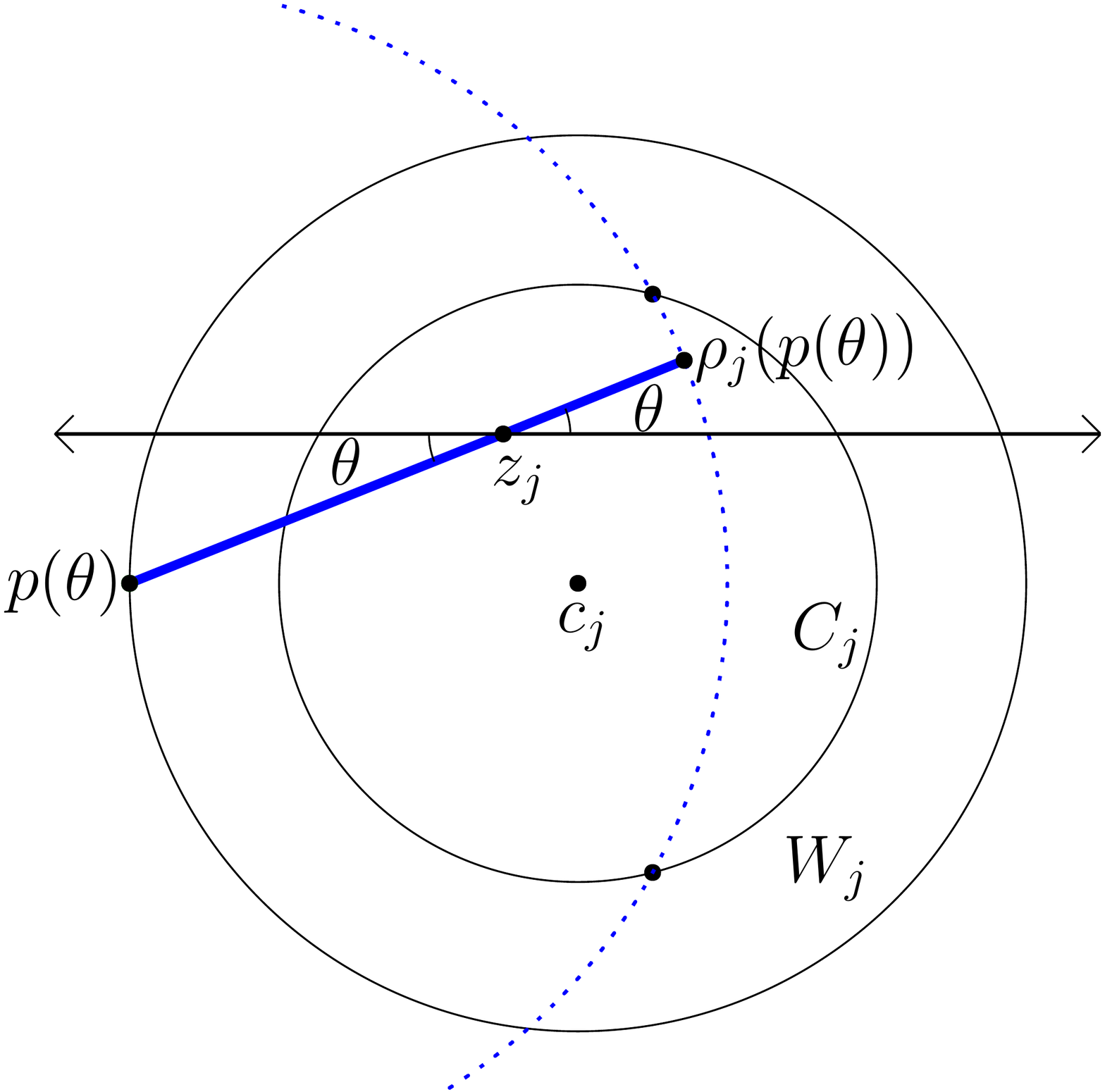}
         \caption{}
     \end{subfigure}
     \hfill
     \begin{subfigure}[b]{0.45\textwidth}
         \centering
         \includegraphics[width=\textwidth]{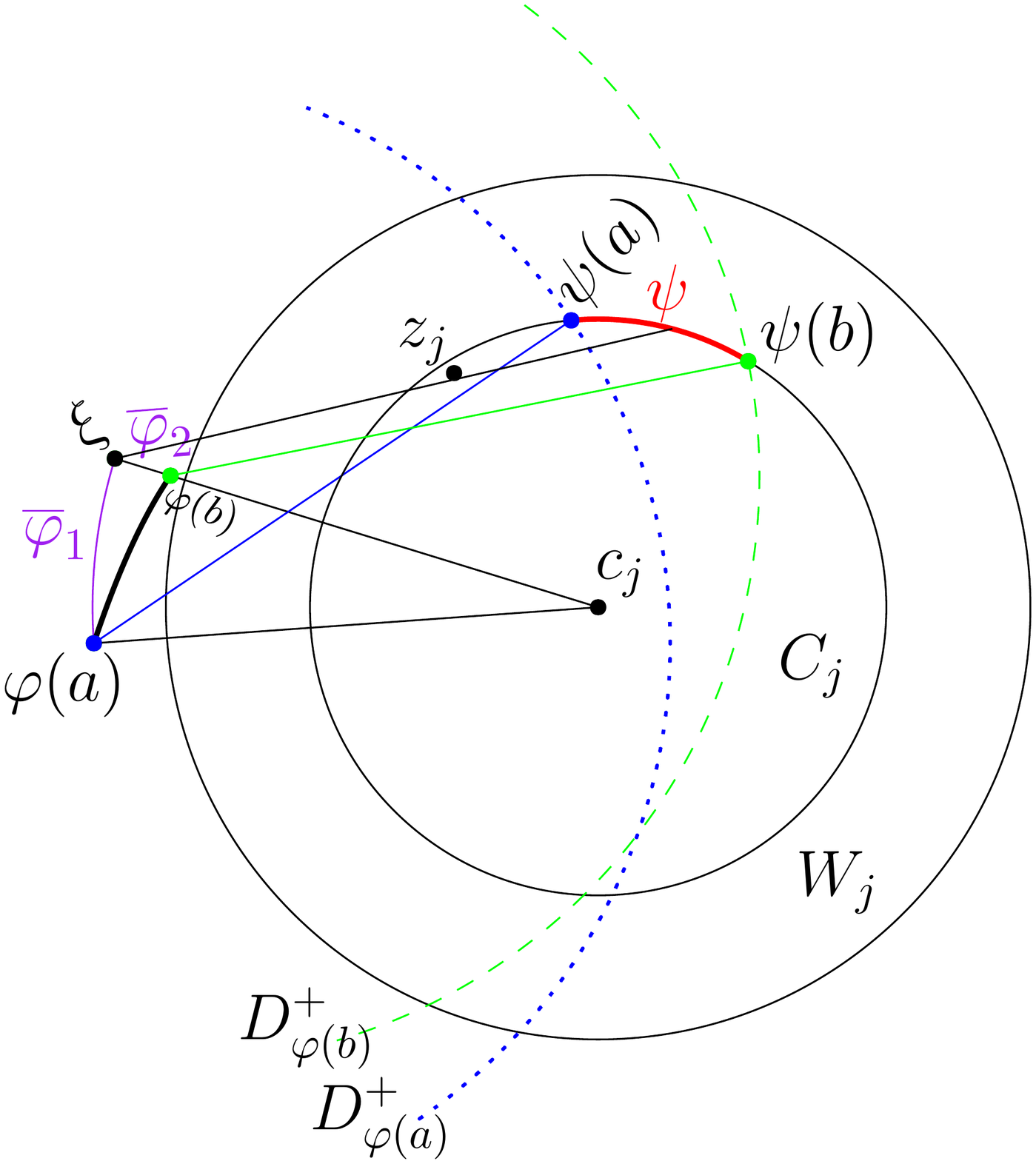}
         \caption{}
     \end{subfigure}
\caption{Illustration of Lemma~\ref{lem:cost-sector} and~\ref{lem:cost-intersect}. 
(a) On the left, the figure shows a sector-type retraction. $p(\theta) = (r(\theta), -\theta)$ and $\rho_j(p(\theta)) = (2-r(\theta), \theta)$. 
(b) On the right, the figure shows an intersection-type retraction. The arc $\psi$ on $\partial C_j$ is the retraction path.}
\label{fig:lipschitz}
\end{figure}


\begin{lemma}
\label{lem:cost-sector}
For a sector-type pathlet $\varphi_k$ of $\Phi_{ji}$, $\ell(\rho_j(\varphi_k)) = O(\ell(\varphi_k))$. 
\end{lemma}

\begin{proof}
For each $p \in \varphi_k$, $\rho_j(p)$ is type sector, i.e.  $\rho_j(p)$ lies on the ray $\overrightarrow{p z_j}$ at distance $2$ from $p$.
In this case, the retraction map $\psi_k = \rho_j(\varphi_k)$ traces a portion of a Conchoid~\cite{schwartz1983piano}.  

We parameterize points on $\varphi \coloneqq \varphi_k$ and $\psi \coloneqq \psi_k$ using polar coordinates, with $z_j$ as the origin. 
Let $\varphi(\theta) = (r(\theta), \theta)$ be a point on $\varphi$, where $\theta$ is the orientation of the point with respect to the $x$-axis (with $z_j$ as the origin). 
Then, $\psi(\theta) = \rho_j(\varphi(\theta)) = (2 - r(\theta) - \theta)$. 
See Figure~\ref{fig:lipschitz}.
Note that $\norm{\varphi'(\theta)}^2 = r^2(\theta) + (r'(\theta))^2$ and $\norm{\psi'(\theta)}^2 = (2 - r(\theta))^2 + (r'(\theta))^2$. 
Since $\varphi$ lies outside $W_j$ and $z_j \in C_j$, we have $r(\theta) \in [1/2, 2]$. 
Therefore, $2 - r(\theta) \leq 3r(\theta)$ and $\norm{\psi'(\theta)} \leq 3\norm{\varphi'(\theta)}$. Hence, 
\[ \ell(\psi) = \int \norm{\psi'(\theta)} d\theta \leq 3 \int \norm{\varphi'(\theta)} d\theta = 3\ell(\varphi). \qedhere\] 
\end{proof}


\begin{lemma}
\label{lem:cost-intersect}
For an intersection-type pathlet $\varphi_k$ of $\Phi_{ji}$, $\ell(\rho_j(\varphi_k)) = O(\ell(\varphi_k))$. 
\end{lemma}

\begin{proof}
Again, we prove the lemma by showing that a Lipschitz condition holds. 
Let $\varphi \coloneqq \varphi_k$.
We parameterize both $\varphi$ and $\rho_j(\varphi)$ in polar coordinates, but with $c_j$ as the origin. 
Let $I = [a, b]$ be the interval over which $\varphi$ is defined.
Let $\varphi(t) = (r(t), \theta(t))$ for $t \in I$. 
We assume that $\varphi$ is sufficiently small (otherwise we divide it into smaller pathlets and argue for each pathlet) so that $\varphi$ both $r$- and $\theta$-monotone.

Set $\Delta \varphi_r = \abs{r(b) - r(a)}$ and $\Delta \varphi_\theta = \abs{\theta(b) - \theta(a)}$. 
Since $\varphi$ lies outside $W_j$, $r(t) \geq 3/2$ for all $t\in [a, b]$. 
W.l.o.g., assume both $r(t)$ and $\theta(t)$ are monotonically non-decreasing. 
We obtain:
\[ \ell(\varphi) =  \int_I \sqrt{r'(t)^2 + r(t) \theta'(t))^2} \geq  \frac{1}{\sqrt{2}} \int_I \Paren{r'(t) + \frac{3}{2}\theta'(t)}\  dt \geq   \frac{1}{\sqrt{2}} (\Delta \varphi_r + \Delta \varphi_\theta). \]

The retraction path $\psi(t)$ varies monotonically on the unit circle $\partial C_z$. Thus, we parameterize $\psi$ by its direction on $\partial C_z$, and $\ell(\psi) = \abs{\int_I \psi'(t)\ dt} = \abs{\psi(b) - \psi(a)}$. 
To bound $\ell(\psi)$, consider the following path from $\varphi(a)$ to $\varphi(b)$, see Figure~\ref{fig:lipschitz}.
Let $\overline{\varphi}_1$ be the arc from $\varphi(a)$ to point $\xi = (r(a), \theta(b))$ along the circle of radius $r(a)$ centered at $c_z$. 
Let $\overline{\varphi}_2$ be the segment $\xi$ to $\varphi(b)$, this is a radial segment on line $\xi c_z$. 
Then, $\ell(\psi) \leq \ell(\rho_j(\overline{\varphi}_1)) + \ell(\rho_j(\overline{\varphi}_2))$. 
Since the radius along $\overline{\varphi}_1$ does not change, $\ell(\rho_j(\overline{\varphi}_2)) = \abs{\theta(b) - \theta(a)} = \Delta \varphi_{\theta}$. 

For a point $p = (r, \theta)$, the orientation of $\rho_j(p)$ is $\theta + \cos^{-1} \Paren{\frac{3 - r^2}{2r}}$ (by the law of cosines, considering triangle $\triangle \rho_j(p) c_z p$). 
Since $\theta$ does not change along $\overline{\varphi}_2$ and $r(a), r(b) \in [3/2, 3]$, we obtain $\ell(\rho_j(\overline{\varphi}_2)) = O(\Delta \varphi_r)$. 

Putting everything together, $\ell(\rho_j(\psi)) = \ell(\psi) = O(\Delta \varphi_r + \Delta \varphi_\theta) = O(\ell(\varphi)).$
\end{proof}

Applying Lemmas~\ref{lem:cost-sector} and~\ref{lem:cost-intersect} to all pathlets of $\Phi_{ij}$, we obtain the following:

\begin{corollary}
\label{cor:outsidewj}
Let $1 \leq i \neq j \leq n$. Let $\Phi_{ij}$ be the portion of $\pi_i$ during the interval $t \in [i-1, i]$ such that $\norm{\pi_i(t) - z_j} \leq 2$ and $\norm{\pi_i(t) - c_j} \geq 3/2$, and let $\Psi_{ji}$ be the retraction of $R_j$ corresponding to $\Phi_{ij}$. 
Then $\ell(\Psi_{ji}) = O(\ell(\Phi_{ij}))$. 
\end{corollary}

\subparagraph{Retraction path inside $\mathbf{W_j}$.} 
\label{subsection:retraction-core}
Recall that $\pi_i$ does not intersect $(\interior\ C_j)$, but possibly travels along $\partial C_j$. 
For a point $p \in \pi_i$, if $p \in \partial C_j$, then $\rho_j(p)$ is the point on $\partial C_j$ diametrically opposite $p$.
Thus, $\ell(\pi_i \cap C_j) = \ell(\rho_j(\pi_i \cap C_j))$.
In the following, we consider only $\pi_i \setminus \partial C_j$. 

\begin{restatable}{lemma}{insidewj}
\label{lem:inside-wj}
For a pathlet $\varphi$ (i.e., a connected subpath) of path $\pi_i$ such that $\varphi \subset W_j \setminus C_j$ for some $j \neq i$, $\ell(\rho_j(\varphi)) = O(\ell(\pi_i \cap A_j))$.
\end{restatable}

\begin{proof}
Since $\varphi \subset W_j$, $\ell(\pi_i \cap A_j) = \Omega(1)$, therefore we only need to argue that $\ell(\rho_j(\varphi))$ is constant. 
We will bound the length of both types of retraction maps (intersection and sector) separately for $\varphi$, and use the sum as an upper bound on the length of the actual retraction map. 

\noindent\textbf{Sector retraction.}
We consider the sector type retraction map. 
Let $z_j$ be the origin and consider polar coordinates. 
Let $\rho_j^s(p)$ be the sector type retraction point with respect to $p$. 
Since $\varphi$ is a subpath of a shortest path in $\F$, 
we can divide $\pi_i \cap W_j$ into at most two pathlets such that each piece is $r, \theta$-monotone. 
Abusing notation, let $\varphi$ be one of these pieces with endpoints $(r_0, \theta_0)$ and $(r_1, \theta_1)$. 

We write the retraction point parameterized by $\theta$ as $(\rho(\theta), \theta)$. 
Using the fact that $\rho(\theta) \leq 2$ for all $\theta$, the arc length of the retraction map is 
\[ \begin{aligned} \ell(\rho_j^s(\varphi)) = \int_{\theta_0}^{\theta_1}
\sqrt{\rho(\theta)^2 + \Paren{\frac{d\rho}{d\theta}}^2} d\theta &\leq \int_{\theta_0}^{\theta_1} \rho(\theta) d\theta + \int_{\theta_0}^{\theta_1} \frac{d\rho(\theta)}{d\theta} d\theta \\
&\leq \rho(\theta_1 - \theta_0) + (\rho(\theta_1) - \rho(\theta_0)) \leq 2(\theta_1 - \theta_0) + 2. 
\end{aligned} \] 

Therefore, $\ell(\rho_j^s(\varphi)) = O(1)$, for each $\varphi$. 
 
\noindent\textbf{Intersection retraction.}
We consider the retraction map defined by an intersection point of $\partial D^+_p$ and $\partial C_j$.
We now let $c_j$ be the origin and consider polar coordinates. 
Let $\rho_j^i(p)$ be the intersection type retraction point closest to $z_j$ with respect to $p$. 
Again, we divide $\pi_i \cap W_j$ into at most two pathlets such that each of them is $r, \theta$-monotone (one pathlet is the portion of $\pi_i$ coming closer to the core $C_j$, and the other moves away from $C_j$). 
Let $\varphi$ be one of the pathlets with endpoints $(r_0, \theta_0)$ and $(r_1, \theta_1)$. 
The retraction point lies on the unit circle $\partial C_j$, and as $\theta$ changes monotonically from $\theta_0$ to $\theta_1$, the retraction point $\rho_j^i(\theta)$ moves monotonically on $\partial C_j$. 
Therefore, $\ell(\rho_j^i(\varphi)) = O(1)$.

Finally, $\ell(\rho_j(\varphi)) \leq \ell(\rho_j^s(\varphi)) + \ell(\rho_j^i(\varphi)) = O(1)$, as claimed.
\end{proof}

Applying Lemma~\ref{lem:inside-wj} to each of (at most two) connected components of $(\pi_i \cap W_j) \setminus C_j$ and combining with Corollary~\ref{cor:outsidewj}, we obtain the following: 

\begin{corollary}
\label{cor:all-retraction}
For $1 \leq i \neq j \leq n$, let $\Delta_{ij}$ be defined as $\Delta_{ij} = \{ t \in [i-1, i] : \norm{\pi_i - z_j} \leq 2 \}$ and let $\pi_{ji} = \pi_j[i-1, i]$. 
Then $\ell(\pi_{ji}) = O(\ell(\pi_i[\Delta_{ij}]))$. 
\end{corollary}

\subparagraph{Cost of Path Ensemble.}
We are now ready to bound the cost of the path ensemble $\pathens$ returned by the algorithm.



\begin{lemma}
\label{lem:tcost}
For an instance $\I$ of optimal MRMP with revolving areas, let $\pathens(\I)$ be the path ensemble returned by the algorithm. 
Then $\cost(\pathens(\I)) = O(1) \cdot \cost^*(\I)$. 
\end{lemma}

\begin{proof}
Set $\pathens = \pathens(\I)$. 
We already argued that $\cost(\overline{\shrtpath}) = O(\cost^*(\I))$, where $\overline{\shrtpath}$ is the path ensemble computed in stage II of the algorithm. 
We thus need to prove $\cost(\pathens) = O(\cost(\overline{\Gamma}))$. 
For a pair $1 \leq i, j \leq n$, let $\pi_{ij} = \pi_i[j-1, j]$. 
By construction, $\ell(\pi_{ii}) = \ell(\overline{\gamma}_i)$. 
For a fixed $i$, 
\[ 
\ell(\pi_i) = \sum_{j = 1}^n \ell(\pi_{ij})=  \ell(\overline{\gamma}_i) + \sum_{j \neq i} \ell(\pi_{ij}) = \ell(\overline{\gamma}_i) + \sum_{j \neq i} O(\ell(\pi_j[\Delta_{ji}])).
 \] 
Where the last equality follows from  Corollary~\ref{cor:all-retraction}. Hence,
\[
\cost(\pathens) = \sum_{j = 1}^n \ell(\pi_i) = \sum_{i = 1}^n \ell(\overline{\gamma}_i) + \sum_{i = 1}^n \sum_{j \neq i} O(\ell(\pi_j[\Delta_{ji}])) = \cost(\overline{\shrtpath}) + \sum_{i = 1}^n \sum_{j \neq i} O(\ell(\pi_j[\Delta_{ji}])). 
\] 
\ 
By definition of $\Delta_{ji}$, $\Delta_{ji} \subseteq [j-1, j]$ and $\pi_j[\Delta_{ji}] \subseteq B_{z_i}$.
Fix a point $x \in \reals^2$. 
Consider a disk $D$ of radius $4$ centered at $x$. If $x \in B_z$ for some $z \in \starts \cup \finals$, then $C_z \subseteq D$. Since cores are pairwise-disjoint (cf Lemma~\ref{lem:ra-ppts}(i)), $D$ can contain at most $16$ core disks
and any $t \in [j-1, j]$ lies in $O(1)$ $\Delta_{ji}$'s. 
Therefore,
\[ \sum_{i \neq j} O(\ell(\pi_j[\Delta_{ji}])) = O(\ell(\pi_j[j-1, j])) = O(\ell(\overline{\gamma}_j)). \] 
Plugging this back in we obtain: $\cost(\pathens) = \cost(\overline{\shrtpath}) + \sum_{j = 1}^n O(\cost(\overline{\gamma}_j)) = O(\cost(\overline{\Gamma})).$ 
\end{proof}

\subsection{Running-time Analysis}
\label{sec:runningtime}
The algorithm has three stages. 
In the first stage, we compute the free space $\F$ with respect to one robot, which takes $O(m \log m)$ time, by computing the Voronoi of $\W$, see the algorithm of~\cite{DBLP:journals/dcg/Yap87}, and  see~\cite{DBLP:journals/dcg/BerberichHKP12} for details. 
In the same stage, we compute a set of shortest paths $\shrtpath$ for $n$ discs, using the algorithm of~\cite{chen2015computing}, taking $O(mn\log m)$ time in total over all robots. 
Each path $\gamma_i \in \shrtpath$ has complexity $O(m)$. 
In stage two of the algorithm, $\gamma_i$ is modified to avoid the core of any occupied revolving area, increasing the complexity of each curve to $O(m+n)$.
In stage three of the algorithm, the deformed paths $\bar{\gamma}_i$ are again edited to include retraction maps in which non-active robots may move within their revolving area.
It suffices to bound the number of breakpoints in the final path $\pi_j$ that correspond to retracted maps. Let $\xi$ be such a breakpoint on $\pi_j$, which is $\rho_{ij}(\bar{\gamma}_i(t))$ for some $t \in [i-1, i]$. There are two cases: (i) the preimage of $\xi$ on $\bar{\gamma}_i$ is a breakpoint of $\bar{\gamma}_i$, or (ii) $\norm{\xi - z_j} = 2$ (i.e., $\bar\gamma_i$ forces $R_j$ to move within the revolving area). We charge both of these breakpoints to $\bar\gamma_i$. Since the preimage of $\xi$ lies in the buffer disk of $R_j$, using a packing argument similar to the proof of Lemma~\ref{lem:tcost} below, we can show that $O(m+n)$ breakpoints are charged to $\bar{\gamma}_i$.

\begin{theorem}
Let $\I = (\W, \starts, \finals, \A)$ be an instance of optimal MRMP with revolving areas, and let $m$ be the complexity of $\W$. If a feasible motion plan of $\I$ exists then a path ensemble $\pathens$ of cost $O(\cost^*(\I))$ can be computed in $O(n(m+n) \log m)$ time. 
\end{theorem}

We conclude this section by noting that since the ordering $\sigma$ (of active robots) is arbitrary, the algorithm can be extended to an online setting where $R_i$ and $(s_i, f_i)$ are given in an online manner (as long as each $s_i, f_i$ given satisfies the revolving area property). 
Our algorithm is $O(1)$-competitive for this setting, i.e., the cost is $O(1)$ times the optimal cost of the offline problem.

\section{Computing a Good Ordering}
\label{sec:ordering}
In the previous section, we proved that the total cost of the path ensemble $\Pi$ is $O(1) \cdot \cost^*(\I)$ irrespective of the order in which the robots moved. 
However, the order in which robots move has a significant impact on how the paths are edited in Stages (II) and (III). 
The increase in cost because of editing may vary between $0$ and $O(nm)$ depending on the ordering (see~\cite{SolomonHalperin2018} for a related argument). 
For a path ensemble $\Pi$ computed by our algorithm, let $\Delta \cost(\Pi) = \cost(\Pi) - \cost(\Gamma)$, which we refer to as the \emph{marginal  cost} of $\Pi$, where $\Gamma$ is the path ensemble computed in Stage (I). 
For a permutation $\sigma$ of $[n]$, let $\Pi_\sigma$ be the path ensemble computed by the algorithm if robots were moved in the order determined by $\sigma$. 
Set $\Delta \cost(\sigma) \coloneqq \Delta \cost(\Pi_\sigma)$. 
Finally, set $\Delta \cost^*(\I) = \min_{\sigma} \Delta \cost(\sigma),$
where the minimum is taken over all permutations of $[n]$. 

Adapting the construction in~\cite{SolomonHalperin2018}, we can show that the problem of determining whether $\Delta \cost^*(\I) \leq L$, for some $L \geq 0$, is NP-hard.
We present an approximation algorithm for computing a good ordering $\sigma$ such that $\Delta \cost(\sigma) = O(\log n \log \log n) \Delta \cost^*(\I)$. 

Our main observation is that $\Delta \cost(\sigma)$, the marginal  cost of an ordering $\sigma$, is decomposable, in the sense made precise below. 
For a pair $i \neq j$, we define $w_{ij} \geq 0$ to be the contribution of the pair $R_i, R_j$ to the marginal cost of an ordering $\sigma$, assuming $i \prec_\sigma j$, i.e., how much the shortest path $\gamma_i$ has to be modified because of $\gamma_j$ and vice-versa assuming $R_i$ is active before $R_j$. 
Note that if $i \prec_{\sigma} j$ then $R_i$ (resp. $R_j$) is at $f_i$ (resp $s_j$) when $R_j$ (resp $R_i)$ is active. There are two components of $w_{ij}^\sigma$: 
\begin{inparaenum}[(i)]
    \item $R_i$ (resp. $R_j$) enters the core $C_{s_j}$ (resp. $C_{f_i}$) in $\gamma_i$ (resp. $\gamma_j$),
    \item retraction motion of $R_j$ (resp. $R_i$) when $R_i$ (resp. $R_j$) enters the buffer disc $B_{s_j}$ (resp. $B_{f_i}$). 
\end{inparaenum}

Let $\phi_{ij}$ (resp. $\phi_{ji}$) be the arc of $\Delta C_{s_j}$ (resp. $\Delta C_{f_i}$) with which $\gamma_i \cap C_{s_j}$ (resp. $\gamma_j \cap C_{f_i}$) is replaced with. 
Then $\alpha_{ij} = \ell(\phi_{ij} + \ell(\phi_{ji}) - \ell(\gamma_i \cap C_{s_j}) - \ell(\gamma_j \cap C_{f_i})$ is the contribution of (i) to $w_{ij}$. 
For (ii), we define $\rho_{ij}^{<}$ (resp. $\rho_{ij}^>$)  be the retraction map of $R_j$ because of $R_i$ when $R_i$ is active before (resp. after) $R_j$. Then $w_{ij} = \alpha_{ij} + \ell(\rho_{ij}^{<}(\bar{\gamma}_i)) + \ell(\rho_{ij}^>(\bar{\gamma}_j)).$
From the previous two components, we have 
$\Delta \cost(\sigma) = \sum_{i, j : i \prec_\sigma j} w_{ij}.$ 

We now reduce the problem of computing an optimal ordering to instance of \emph{weighted feedback-arc-set} (FAS) problem. 
Given a directed graph with weights on the edges, $G = (V, E), w: E \rightarrow \reals_{\geq 0}$, a feedback arc set $F$ is a subset of edges of $G$ whose removal makes $G$ a directed acyclic graph. 
The weight of $F, w(F)$, is $\sum_{e \in F} w(e)$. 
The FAS problem asks to compute an FAS of the smallest weight. 
It is known to be NP-complete. 

Given an MRMP-RA instance $\I = (\W, \starts, \finals, \A)$, we first compute $\Gamma$ as in stage (I) of the algorithm. 
Next, for each pair $i, j \in [n]$, we construct a directed graph as follows. 
$G = (V, E)$ is a complete directed graph with $V = [n]$, one representing each robot, $E = \{ i \rightarrow j : 1 \leq i \neq j \leq n \}$, $w(i \rightarrow j) = w_{ij}$. It can be shown that each feedback arc set $F$ of $G$ induces an ordering $\sigma_F$ on $[n]$, and vice versa. Furthermore, $w(F) = \Delta \cost(\sigma_F)$. 
Even et al.~\cite{even1998approximating} have described a polynomial-time $O(\log n \log \log n)$-approximation algorithm for the FAS problem. By applying their algorithm to $G$, we obtain the following. 

\begin{theorem}
Let $\I = (\W, \starts, \finals, \A)$ be an instance of optimal MMP with revolving areas, and let $m$ be the complexity of $\W$. Let the optimal order of execution of paths be $\sigma^*$. An ordering $\sigma$ with $\Delta \cost(\sigma) = O(\log n \log \log n) \Delta \cost(\sigma^*)$ can be computed in polynomial time in $n$ and $m$. 
\end{theorem}
\section{Conclusion}
In this work, we presented the first constant-factor approximation algorithm for computing a feasible weakly-monotone motion plan to minimize the sum of distances traveled. Additionally, the algorithm can be extended to an online setting where the polygonal environment is fixed, but the initial and final positions of the robots are specified in an online manner. On the hardness side, we prove that minimizing the total traveled distance, even with the restriction of a weakly-monotone motion plan, is APX-hard.

There are several interesting open questions. The first is whether the constant factor approximation presented in this work can be improved; another is whether there are instances in which the separation bounds for revolving areas are not required or can be tightened. There are other objectives to consider; instead of the sum of distances objective, one can consider the makespan (latest arrival time), where little is known even for a small number of discs in the presence of obstacles. 

\bibliographystyle{plainurl}
\bibliography{ref}

\begin{thebibliography}{10}

\bibitem{DBLP:journals/tase/AdlerBHS15}
Aviv Adler, Mark de~Berg, Dan Halperin, and Kiril Solovey.
\newblock Efficient multi-robot motion planning for unlabeled discs in simple
  polygons.
\newblock {\em {IEEE} Trans Autom. Sci. Eng.}, 12(4):1309--1317, 2015.

\bibitem{BanyassadyEtAl.SoCG.2022}
Bahareh Banyassady, Mark de~Berg, Karl Bringmann, Kevin Buchin, Henning Fernau,
  Dan Halperin, Irina Kostitsyna, Yoshio Okamoto, and Stijn Slot.
\newblock {Unlabeled Multi-Robot Motion Planning with Tighter Separation
  Bounds}.
\newblock In {\em 38th International Symposium on Computational Geometry
  (SoCG)}, 2022.

\bibitem{DBLP:journals/dcg/BerberichHKP12}
Eric Berberich, Dan Halperin, Michael Kerber, and Roza Pogalnikova.
\newblock Deconstructing approximate offsets.
\newblock {\em Discret. Comput. Geom.}, 48(4):964--989, 2012.

\bibitem{DBLP:conf/fun/BrunnerCDHHSZ21}
Josh Brunner, Lily Chung, Erik~D. Demaine, Dylan~H. Hendrickson, Adam
  Hesterberg, Adam Suhl, and Avi Zeff.
\newblock 1 {X} 1 rush hour with fixed blocks is {PSPACE}-complete.
\newblock In {\em 10th International Conference on Fun with Algorithms}, volume
  157, pages 7:1--7:14, 2021.

\bibitem{chen2015computing}
Danny~Z Chen and Haitao Wang.
\newblock Computing shortest paths among curved obstacles in the plane.
\newblock {\em ACM Transactions on Algorithms}, 11(4):1--46, 2015.

\bibitem{DBLP:conf/icra/DayanSPH21}
Dror Dayan, Kiril Solovey, Marco Pavone, and Dan Halperin.
\newblock Near-optimal multi-robot motion planning with finite sampling.
\newblock In {\em {IEEE} International Conference on Robotics and Automation},
  pages 9190--9196, 2021.

\bibitem{demaine2019coordinated}
Erik~D. Demaine, S{\'{a}}ndor~P. Fekete, Phillip Keldenich, Henk Meijer, and
  Christian Scheffer.
\newblock Coordinated motion planning: Reconfiguring a swarm of labeled robots
  with bounded stretch.
\newblock {\em {SIAM} Journal on Computing}, 48(6):1727--1762, 2019.

\bibitem{even1998approximating}
Guy Even, J~Seffi Naor, Baruch Schieber, and Madhu Sudan.
\newblock Approximating minimum feedback sets and multicuts in directed graphs.
\newblock {\em Algorithmica}, 20(2):151--174, 1998.

\bibitem{geft2021complexity}
Tzvika Geft and Dan Halperin.
\newblock Tractability frontiers in multi-robot coordination and geometric
  reconfiguration, 2021.
\newblock \href {http://arxiv.org/abs/2104.07011} {\path{arXiv:2104.07011}}.

\bibitem{GH2021Refined}
Tzvika Geft and Dan Halperin.
\newblock Refined hardness of distance-optimal multi-agent path finding.
\newblock In {\em 21st International Conference on Autonomous Agents and
  Multiagent Systems, {AAMAS}}, pages 481--488, 2022.

\bibitem{hopcroft1984complexity}
John~E Hopcroft, Jacob~Theodore Schwartz, and Micha Sharir.
\newblock On the complexity of motion planning for multiple independent
  objects; {PSPACE}-hardness of the "warehouseman's problem".
\newblock {\em The International Journal of Robotics Research}, 3(4):76--88,
  1984.

\bibitem{karaman2011sampling}
Sertac Karaman and Emilio Frazzoli.
\newblock Sampling-based algorithms for optimal motion planning.
\newblock {\em International Journal of Robotics Research}, 30(7):846--894,
  2011.

\bibitem{o1985retraction}
C~O'Ddnlaing and CK~Yap.
\newblock A retraction method for planning the motion of a disc.
\newblock {\em J. Algorithms}, 6:104--111, 1985.

\bibitem{DBLP:journals/cacm/Salzman19}
Oren Salzman.
\newblock Sampling-based robot motion planning.
\newblock {\em Commun. {ACM}}, 62(10):54--63, 2019.

\bibitem{schwartz1983piano}
Jacob~T Schwartz and Micha Sharir.
\newblock On the piano movers' problem: {III}. coordinating the motion of
  several independent bodies: The special case of circular bodies moving amidst
  polygonal barriers.
\newblock {\em The International Journal of Robotics Research}, 2(3):46--75,
  1983.

\bibitem{DBLP:journals/arobots/ShomeSDHB20}
Rahul Shome, Kiril Solovey, Andrew Dobson, Dan Halperin, and Kostas~E. Bekris.
\newblock {dRRT\({}^{\mbox{*}}\)}: Scalable and informed asymptotically-optimal
  multi-robot motion planning.
\newblock {\em Auton. Robots}, 44(3-4):443--467, 2020.

\bibitem{SolomonHalperin2018}
Israela Solomon and Dan Halperin.
\newblock Motion planning for multiple unit-ball robots in {$\mathbb{R}^d$}.
\newblock In {\em Workshop on the Algorithmic Foundations of Robotics, {WAFR}},
  pages 799--816, 2018.

\bibitem{DBLP:journals/ijrr/SoloveyH16}
Kiril Solovey and Dan Halperin.
\newblock On the hardness of unlabeled multi-robot motion planning.
\newblock {\em Int. J. Robotics Res.}, 35(14):1750--1759, 2016.

\bibitem{DBLP:conf/icra/SoloveyJSFP20}
Kiril Solovey, Lucas Janson, Edward Schmerling, Emilio Frazzoli, and Marco
  Pavone.
\newblock Revisiting the asymptotic optimality of {RRT}.
\newblock In {\em 2020 {IEEE} International Conference on Robotics and
  Automation, {ICRA} 2020, Paris, France, May 31 - August 31, 2020}, pages
  2189--2195. {IEEE}, 2020.

\bibitem{DBLP:conf/rss/SoloveyYZH15}
Kiril Solovey, Jingjin Yu, Or~Zamir, and Dan Halperin.
\newblock Motion planning for unlabeled discs with optimality guarantees.
\newblock In {\em Robotics: Science and Systems}, 2015.

\bibitem{stern2019multiagent}
Roni Stern, Nathan~R. Sturtevant, Ariel Felner, Sven Koenig, Hang Ma, Thayne~T.
  Walker, Jiaoyang Li, Dor Atzmon, Liron Cohen, T.~K.~Satish Kumar, Roman
  Bart{\'{a}}k, and Eli Boyarski.
\newblock Multi-agent pathfinding: Definitions, variants, and benchmarks.
\newblock In {\em Proc. 12th International Symposium on Combinatorial Search},
  pages 151--159, 2019.

\bibitem{DBLP:journals/arobots/TurpinMMK14}
Matthew Turpin, Kartik Mohta, Nathan Michael, and Vijay Kumar.
\newblock Goal assignment and trajectory planning for large teams of
  interchangeable robots.
\newblock {\em Auton. Robots}, 37(4):401--415, 2014.

\bibitem{vazirani2001approximation}
Vijay~V Vazirani.
\newblock {\em Approximation algorithms}.
\newblock Springer, 2001.

\bibitem{DBLP:journals/dcg/Yap87}
Chee{-}Keng Yap.
\newblock An {O} (n log n) algorithm for the voronoi diagram of a set of simple
  curve segments.
\newblock {\em Discret. Comput. Geom.}, 2:365--393, 1987.

\end{thebibliography}


\end{document}